\documentclass{article} 
\usepackage{arxiv_conference,times}

\usepackage{amsmath,amsfonts,bm}

\def\eqref#1{equation~\ref{#1}}

\def\1{\bm{1}}

\def\vmu{{\bm{\mu}}}
\def\vtheta{{\bm{\theta}}}

\def\vg{{\bm{g}}}

\def\vu{{\bm{u}}}
\def\vv{{\bm{v}}}

\def\mH{{\bm{H}}}

\def\mP{{\bm{P}}}

\DeclareMathAlphabet{\mathsfit}{\encodingdefault}{\sfdefault}{m}{sl}
\SetMathAlphabet{\mathsfit}{bold}{\encodingdefault}{\sfdefault}{bx}{n}

\def\gB{{\mathcal{B}}}

\def\gD{{\mathcal{D}}}

\def\gP{{\mathcal{P}}}

\newcommand{\E}{\mathbb{E}}

\usepackage{graphicx}
\usepackage[T1]{fontenc}
\usepackage[utf8]{inputenc}
\usepackage{booktabs}
\usepackage{amsthm}
\usepackage{mathrsfs}
\usepackage{subcaption}
\usepackage{wrapfig}
\usepackage{cancel}

\usepackage{hyperref}
\usepackage{url}
\usepackage{xcolor}
\definecolor{mydeepblue}{RGB}{26,64,149}
\definecolor{mygreen}{RGB}{37,145,60}
\definecolor{myorange}{RGB}{217,113,26}
\definecolor{mypurple}{RGB}{117,60,204}
\hypersetup{
    colorlinks=true,
    linkcolor=myorange,
    citecolor=mygreen,
    urlcolor=myorange,
    filecolor=myorange
}

\title{A Journey to the Edge of Stability}

\author{Jaerin Lee \& Kyoung Mu Lee\\
Computer Vision Lab, ASRI\\
Seoul National University\\
Seoul 08826, Korea \\
\texttt{\{ironjr,kyoungmu\}@snu.ac.kr}
}

\newtheorem{theorem}{Theorem}[section]

\newtheorem{proposition}[theorem]{Proposition}
\newtheorem{corollary}[theorem]{Corollary}

\iclrfinalcopy 
\begin{document}

\maketitle

\begin{abstract}
It has recently been found that deep learning often occurs at the ``edge of stability (EoS),'' where the maximum Hessian eigenvalue of the model is stabilized at a value reciprocal to the learning rate.
However, what happens \emph{before} we reach that regime?
We fix a deep learning problem and vary first order optimization methods with dense learning rate sweeps.
We then track the characterizing quantities of a learning trajectory: the loss, the sharpness, and the alignment between consecutive gradients.
To our surprise, if we scale the learning rate by the dc gain of the optimizer, these traces from the sweeps from different optimizers almost perfectly overlap across a large range of learning rates.
The dc-normalized optimizers have another role that only becomes apparent in high learning rates: they select when the sharpness value detaches from this universal curve and enters the edge of stability.
Upon this discovery, we specify three distinct regimes with respect to the dc-adjusted learning rate: the \emph{low-LR regime} where the trajectory is nearly insensitive to the optimizer, the \emph{high-LR regime}, where the optimizer governs the sharpness according to the EoS reciprocal rule, and the in-between \emph{mid-LR regime} where so-called \emph{progressive sharpening} originates independently of the optimizer.
This distinguishes the role of the optimizer, the learning rate, and the model in shaping the learning progress.

\end{abstract}

\section{Introduction}
\label{sec:introduction}
Modern deep learning is mainly driven by so-called first order methods, i.e., the (stochastic) gradient descent and its variants.
Under gradient descent, the parameter $\vtheta$ moves according to a scaled gradient $\vtheta \leftarrow \vtheta - \eta \nabla_{\vtheta} L(\vtheta)$, which can be viewed as a discretization of a continuous-time gradient flow $\dot \vtheta = - \nabla_{\vtheta} L(\vtheta)$ with a sampling period $\eta$.
However, it is recently observed that the difference between the two dynamics—the high-order terms, or the \emph{aliasing} terms if we want to take the signal processing perspective—becomes significant in practical settings where the learning rate $\eta$ operates near its maximally stable value.
This aliasing effect develops into a unique characteristic feature of the discrete-time gradient descent, called the \emph{edge of stability} (EoS)~\citep{cohen2021gradient}.
For deterministic full-batch gradient descent, the maximum sharpness $S := \lambda_{\max} (\nabla_{\vtheta}^2 L(\vtheta))$ of the loss function $L(\vtheta)$ is often attracted and stabilized at the value 
$S \approx 2/\eta$,
under which the gradient descent is certain to produce monotonic loss decrease if the loss is quadratic~\citep{ortega1970iterative,bertsekas1999nonlinear}.
Hence the name \emph{edge of stability}.

Similar phenomenon has been observed for preconditioned and momentum-augmented gradient descent~\citep{cohen2021gradient,cohen2023adaptive}.
\citet{andreyev2024edge,andreyev2026momentum} have further generalized this phenomenon to more practical mini-batch settings, incorporating stochastic gradient descent (SGD) and SGD with momentum~\citep{polyak1964some,nesterov1983method}.
Later, \citet{lee2026road} have added a missing term in the relationship $S \approx 2/\eta$ by treating the optimizer as a gradient filter to define its transfer function $Q(z)$, dc gain $Q(1)$, and root locus escape gain $\Gamma(Q)$:
\begin{equation}\tag{{\color{blue!50}$\spadesuit$}}
\label{eq:edge_of_stability_lee}
\renewcommand{\boxed}[1]{\colorbox{lightgray!15}{\ensuremath{#1}}}
\boxed{\ %
\displaystyle
S \approx \frac{2}{\eta Q(1) \Gamma(Q)}.
}
\end{equation}
Overall, it is now well-established that there are essentially discrete-time behaviors of the first order methods in deep learning, summarized in the name: \emph{edge of stability} (EoS)~\citep{cohen2021gradient}.
But how does this phenomenon emerge from the continuous-time gradient flow?
What happens before the edge of stability?
What are the roles of the optimizer, the learning rate, and the model in producing this phenomenon?
When, how, and how much is each of them reponsible in steering the learning progress?
Our goal is to answer these questions by inspecting the learning progress of a \emph{single} deep learning problem solved with various optimizers and learning rates.

We focus on a single large-batch stochastic scenario of learning a CIFAR-10 dataset~\citep{krizhevsky2009learning}.
Inspired by previous work~\citep{andreyev2026momentum,lee2026road}, we devise various optimizers, including Heavy Ball (SGDM)~\citep{polyak1964some}, Nesterov's momentum (SGDN)~\citep{nesterov1983method}, two-pole cascade optimizers~\citep{lee2024grokfast}, and quasi-hyperbolic momentum (QHM) optimizers~\citep{ma2018quasi}, to vary the optimizer's transfer function $Q(z)$, the gains $Q(1)$ and $\Gamma(Q)$, and the unit circle escape angle in the root locus plot.
For each of the optimizer, we densely sweep the raw learning rate $\eta$ in the wide range from $10^{-7}$ to $10^{0}$ in log scale.
The results are beyond our expectations.
We observe three distinct regimes separated in the \emph{effective learning rate} $h := \eta Q(1)$, each of which is characterized by a unique nonlinear relationship between the learning rate $\eta$, the stabilized sharpness $S$, the loss $L$, and the optimizer's transfer function $Q(z)$.
We dub these regimes as the \emph{low-LR softening regime}, \emph{mid-LR sharpening regime}, and \emph{high-LR edge regime}, respectively.
In the follwing sections, we will elaborate on the characteristics of the learning dynamics in each of these regimes in detail.
In specific, our contributions are:

\paragraph{Separable roles of problem geometry and optimizer.}
Training with a single deep learning problem across a broad family of optimizers reveals a remarkably consistent pre-edge traces of the loss $L$ and the sharpness $S$.
The effect of changing an optimizer only takes place near and on the edge of stability.
The problem geometry and the optimizer memory play time-scale-separable roles.

\paragraph{Mechanism of optimizers determining the edge of stability.}
We show that the relationship of type~(\ref{eq:edge_of_stability_lee}) universaly occurs in diverse stochastic first order methods, including multi-pole~\citep{lee2024grokfast} and pole-zero~\citep{ma2018quasi} optimizers.
To our knowledge, this is the first empirical evidence that the optimizer root locus determines the edge of stability in stochastic settings.

\paragraph{Rethinking optimizers as gain-shape compounds.}
The Low-LR traces of the model geometry are insensitive to the optimizer, coinciding in the \emph{effective} lr $h := \eta Q(1)$ coordinate.
This suggests that an optimizer $Q$ is a compound of separable mechanistic parts.
Its dc gain $Q(1)$ is more sensible to be absorbed into lr, together governing the universal trajectories in the Low-LR regimes.
Whereas, the temporal reshaping kernel $\bar Q(z) := Q(z)/Q(1)$ determines the edge of stability.

\paragraph{Mechanism of progressive sharpening.}
We empirically show what determines the mysterious mechanism of \emph{progressive sharpening} reported in the EoS literature~\citep{cohen2021gradient,damian2023self}.
The manifestation of this phenomenon is determined by the model geometry, independently of the optimizer. However, its realization depends on the effective lr $h = \eta Q(1)$.

\section{Observations}
\label{sec:preliminary}
\paragraph{Problem formulation.}
A \emph{deep learning problem} is a tuple
\begin{equation}
\label{eq:problem_formulation}
\mathcal{P} = (\gD, \gP_b, \vtheta, L(\vtheta), \vtheta_0, \eta, Q),
\end{equation}
where $\xi \in \gD$ is the dataset, $\{\xi_i\}_{i=1}^b \sim \gP_b$ is the sampling distribution of size $b$, $\vtheta$ is the parameter with a topology, $L(\cdot; \vtheta)$ is the model-loss compound, i.e., the manifestation of $\vtheta$ on the scalar objective in response to a subset of $\gD$, $\vtheta_0$ is the initialization, $\eta$ is the learning rate scalar, and $Q$ is the optimizer \emph{filter}~\citep{lee2024grokfast}.
We write the empirical risk $L(\vtheta) := \E_{\gB \sim \gP_b} [L(\gB; \vtheta)]$ and its gradient $\vg(\vtheta) := \nabla_{\vtheta} L(\vtheta)$ and Hessian $\mH(\vtheta) := \nabla^2_{\vtheta} L(\vtheta)$.
The sharpness is defined as $S(\vtheta) := \lambda_{\max}(\mH(\vtheta))$.
Optimizer preconditioners~\citep{duchi2011adaptive,tieleman2012lecture,kingma2015adam} separate sharpnesses into preconditioned and unpreconditioned variants, complicating the analysis.
We treat the effect of preconditioning out of scope of this work, and focus on the memory structure of the optimizer $Q$.

An optimizer $Q$ is an optionally stateful causal linear filter on a stochastic gradient stream $\vg_{\le t}$.
Filter means that its operation summarizes input sequences through a linear convolution with its impulse response $Q_t$.
That is, $Q * \vg_{\le t} = \sum_{j=0}^{\infty} Q_j \vg_{t-j}$.
Under a first order method, we update $\vtheta$ as
\begin{equation}
\label{eq:gradient_filter}
\vtheta_{t+1} \;=\; \vtheta_t - \eta (Q * \vg_{\le t})_t \;.
\end{equation}
This class includes widely used SGD~\citep{robbins1951stochastic} and the two most famous variants of momentum: Heavy Ball~\citep{polyak1964some} and Nesterov~\citep{nesterov1983method}.
To study optimizers with an emphasis on their signal filtering mechanism, we additionally consider the following two classes of optimizers: dual-momentum cascade of Grokfast~\citep{lee2024grokfast}, and quasi-hyperbolic momentum (QHM)~\citep{ma2018quasi}.
These two additional classes correspond to well-used linear filters: two-pole and pole-zero filters, respectively.
Treating optimizers as filters lets us study them in the frequency dual domain, i.e., $z$-domain: $Q(z) = \sum_{t=0}^{\infty} Q_t z^{-t}$.
We write $Q(1)$ for the dc gain of this filter.
The \emph{effective (dc-matched) learning rate} is defined as $h := \eta Q(1)$.

\paragraph{Optimizer root locus.}
Treating gradient-based optimizers as filters unlocks a powerful tool to study its behavior: a root locus~\citep{ogata1995discrete}.
Following \citet{lee2026road}, we consider a dynamics along an eigenpair $(\lambda, \vv)$ of $\mH$.
Let the parameter $\vtheta$ in this coordinate be $x := \vv^\top (\vtheta - \vtheta_\star)$, where $\vtheta_\star$ is the local minimum along the direction $\vv$.
Following the derivation in \citet{lee2026road}, we get the characteristic equation of the transfer function $Q(z)$
\begin{equation}
\label{eq:eigencomponent_transfer_function}
z - 1 + \eta \lambda Q(z) \;=\; 0.
\end{equation}
We further define a dc-normalized transfer function $\bar Q(z) := Q(z) / Q(1)$ to separate gain control from the filter shape~\citep{lee2025greedy}.
Focusing on the eigenpair $(S, \vv_\text{max})$ of the maximum eigenvalue (sharpness) $S = \lambda_{\max}(\mH)$, let $k := hS$.
Then the characteristic equation becomes
\begin{equation}\tag{{\color{red!50}$\clubsuit$}}
\label{eq:eigencomponent_transfer_function_normalized}
\renewcommand{\boxed}[1]{\colorbox{lightgray!15}{\ensuremath{#1}}}
\boxed{\ %
z - 1 + k \bar Q(z) \;=\; 0.
}
\end{equation}
This equation governs the relationship between the effective learning rate $h$ (gain), the sharpness $S$, and the filter shape $\bar Q(z)$.
In practical scenarios, we first define the filter shape $\bar Q(z)$, i.e., the optimizer family, such as SGD or SGDM, and then tune the gain $h = \eta Q(1)$.

\paragraph{Experimental setup.}
Training dynamics are determined by multitudes of components.
Our main focus is on the joint effect of the optimizer $Q$, the learning rate $\eta$, and the geometry of the loss landscape $L(\vtheta)$.
Therefore, the remaining components are fixed throughout the analysis.
We fix the task to CIFAR-10~\citep{krizhevsky2009learning} and consider three representative architectures: a two-layer tanh MLP having our special attention, and a four-layer tanh CNN and a small three block ViT~\citep{dosovitskiy2021image}.
The above number of layers did not count the common linear classification head.
We fix the batch size to 4096 for all experiments, which is in the \emph{large batch regime} specified by \citet{andreyev2026momentum}.
A diverse set of optimizers are considered in this study, classified into pole-zero structures in their transfer function of $Q(z)$: SGD, SGDM, SGDN, QHM, and two-pole cascade (Grokfast).
In total of 10 optimizers are considered.
Using each of these filters, we train the same model from a same initialization point.
This is to fix the initial geometry and study the trajectory deviation caused by the optimizer.
We vary the learning rate over a wide range, roughly equally spaced on a log scale: $\{1.0, 1.4, 2.0, 3.0, 5.0, 7.0\} \times 10^{\{-7,-6,\ldots,0\}}$.

\paragraph{Measurements.}
Throughout 200k iterations, we collect the sharpness $S$, the loss $L$, and two types of flow smoothness measures:
a \emph{consecutive gradient alignment} $c_t$ as the cosine similarity, and a \emph{first order alias} $\rho_t$ as the normalized difference of the consecutive gradient pair:
\begin{equation}
\label{eq:consecutive_gradient_alignment}
c_t \;:=\; \frac{\langle \vg_t, \vg_{t+1} \rangle}{\|\vg_t\| \|\vg_{t+1}\|} \;\in\; [-1, 1], \qquad
\rho_t \;:=\; \left| \frac{1}{2} - \frac{Q'(1)}{Q(1)} \right| \frac{\|\vg_{t+1} - \vg_{t}\|}{\|\vg_t\|} \;\in\; [0, \infty),
\end{equation}
where time $\tau := ht = t \eta Q(1)$ is defined as the product of the effective learning rate $h$, i.e., the sampling period, and the algorithm iteration index $t$.
This gets the governing gradient flow normalized as $\mathrm d\vtheta / \mathrm d\tau = -\vg$.
The alignment $c_t$ measures the effect of Hessian-induced gradient rotation.
In Proposition~\ref{prop:first_order_alias}, the alias $\rho_t$ is shown to be equivalent to the relative contribution of the higher order terms in the Taylor expansion up to $O(h^2)$.
Among these, the sharpness $\lambda_{\max}(\mH_t)$ may oscillate wildly, especially at the edge of stability.
However, taking its median over a fixed window size of 51 samples greatly stabilizes the value to let us study its behavior more clearly.
We use this scalar \emph{stabilized sharpness} as the representative curvature signature.

\begin{figure}[t]
\centering
\begin{minipage}[t]{0.6425\linewidth}
    \vspace{0pt}
    \begin{subfigure}[t]{\textwidth}
        \centering
        \includegraphics[width=\linewidth]{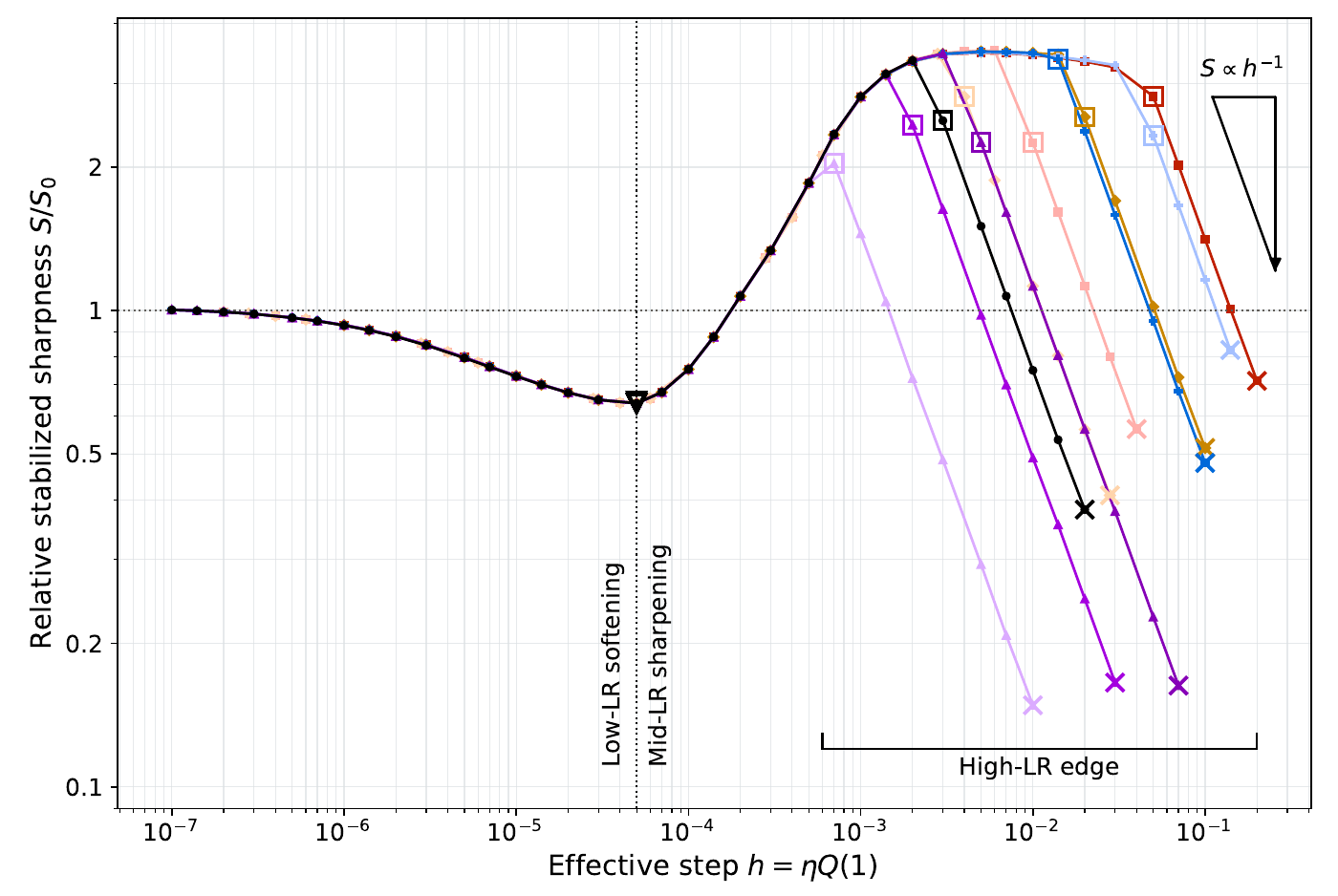}
        \caption{\small Stabilized sharpness $S$ vs. effective lr $h = \eta Q(1)$.}
        \vspace{3pt}
    \end{subfigure}
    \begin{subfigure}[t]{\textwidth}
        \centering
        \includegraphics[width=\linewidth]{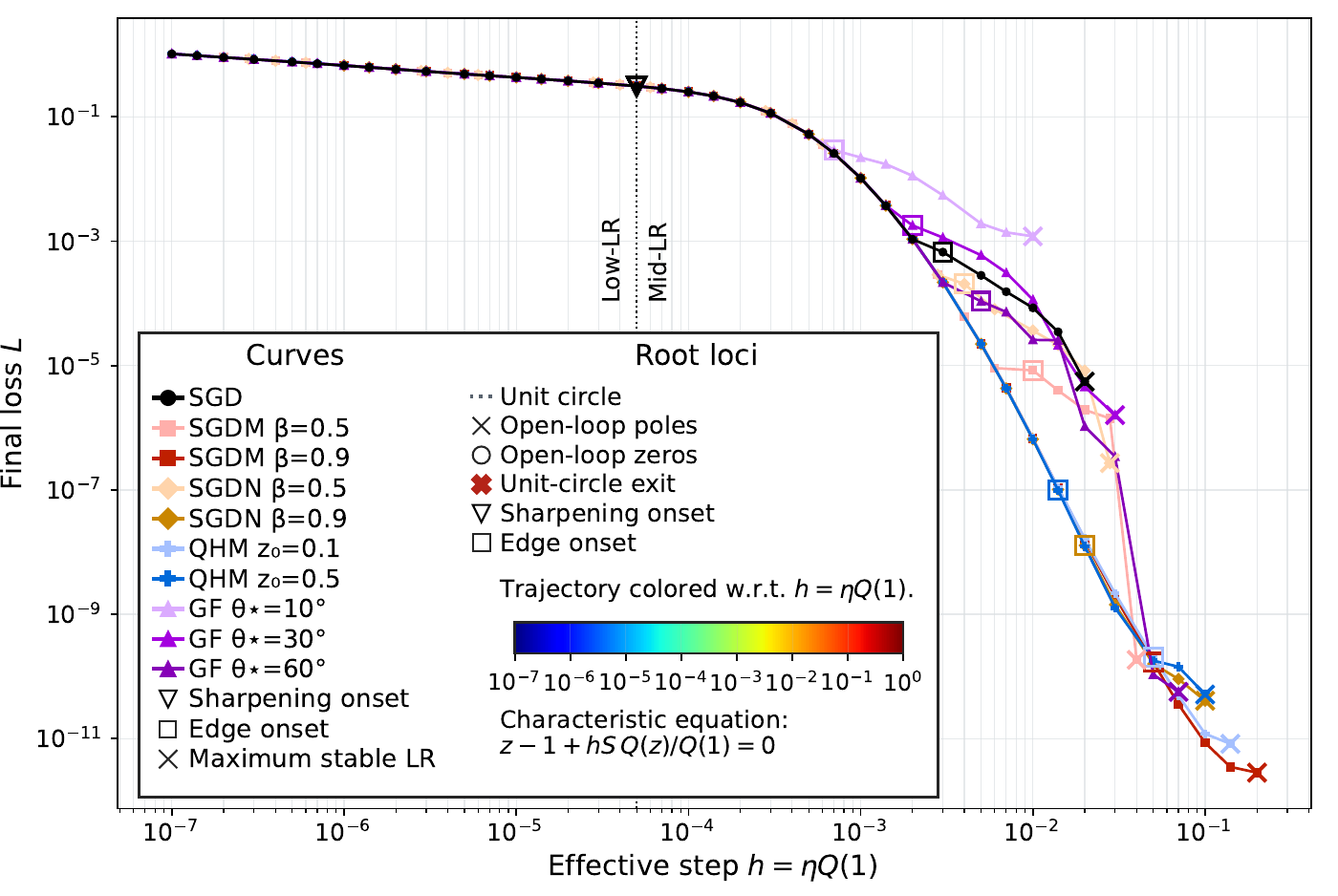}
        \caption{\small Final loss $L$ vs. effective lr $h = \eta Q(1)$.}
    \end{subfigure}
\end{minipage}%
\hfill
\begin{minipage}[t]{0.3425\linewidth}
    \vspace{0pt}
    \begin{subfigure}[t]{\textwidth}
        \centering
        \includegraphics[width=\linewidth]{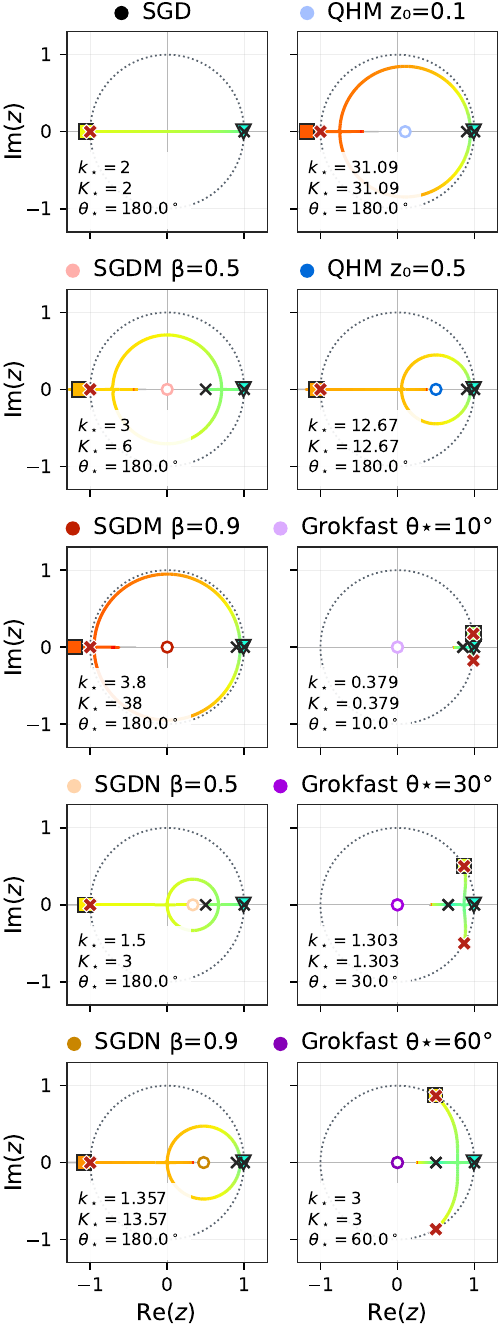}
        \caption{\small Full root locus plots.}
    \end{subfigure}
\end{minipage}
\vspace{0pt}
\begin{minipage}{0.4975\linewidth}
    \vspace{0pt}
    \begin{subfigure}[t]{\textwidth}
        \centering
        \includegraphics[width=\linewidth]{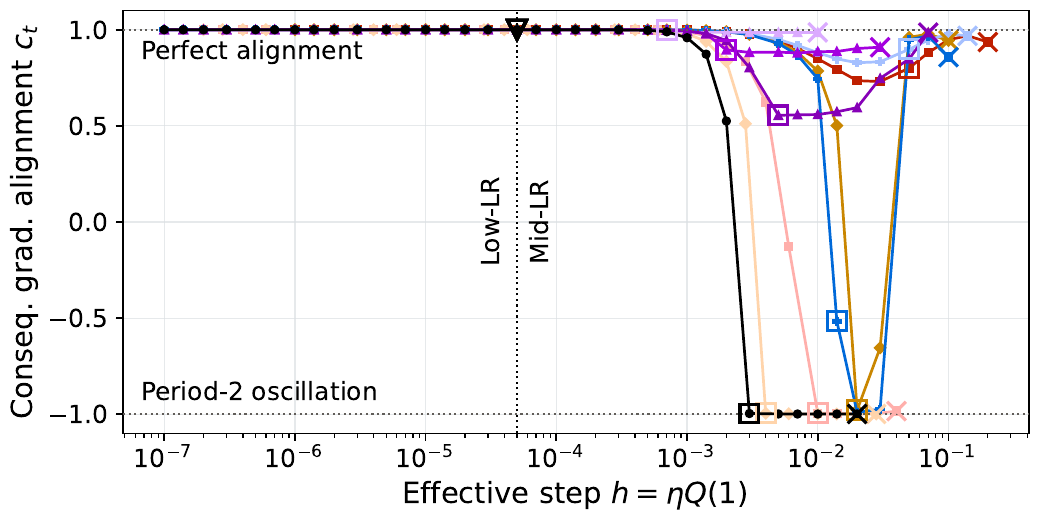}
        \caption{\small Consecutive gradients alignment $c_t$.}
        \vspace{3pt}
    \end{subfigure}
\end{minipage}%
\hfill
\begin{minipage}{0.4975\linewidth}
    \vspace{0pt}
    \begin{subfigure}[t]{\textwidth}
        \centering
        \includegraphics[width=\linewidth]{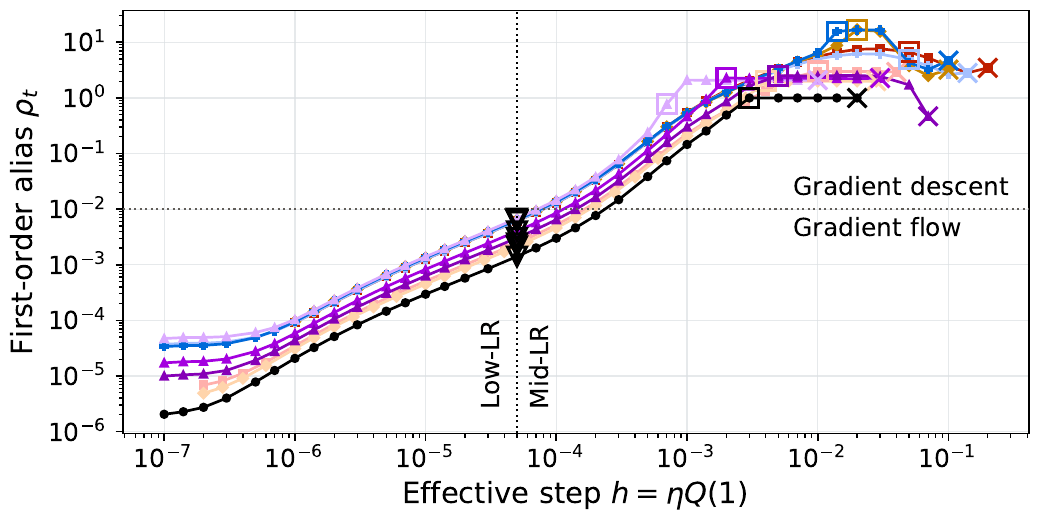}
        \caption{\small First order alias $\rho_t$ of GD vs. GF.}
        \vspace{3pt}
    \end{subfigure}
\end{minipage}
\vspace{-1em}
\caption{%
\small
\textbf{Main results.}
All results drawn in the same figure to visualize coincidence.
Each line connects a learning rate sweep over a single optimizer.
In the effective lr $h$ coordinate, the operating range of this learning system is naturally partitioned into three distinct regimes.
In the low and mid-LR regimes, the results coincide in an optimizer-independent curves.
In the edge regime, the optimizer plays a role in determining the edge of stability, marked by the strict oblique lines $S \propto 1/h$ in Figure~\ref{fig:main_results}a.
Consecutive gradient alignment $c_t$ and first order alias $\rho_t$ marks the natural transition boundaries between these three regimes, marked by the triangle (between low and mid-LR) and square (between mid-LR and edge) markers.
}
\label{fig:main_results}
\end{figure}

\paragraph{Main results.}
We now summarize our main findings in Figure~\ref{fig:main_results} and~\ref{fig:geometry_response_of_cnn_and_vit}.
Figures~\ref{fig:main_results}a and \ref{fig:main_results}b plot two geometric quantities, the stabilized sharpness $S$, and the final loss $L$, respectively, with respect to the aligned x-axis of effective learning rate $h = \eta Q(1)$ in logarithmic scale.
Each dot summarizes a full 200k training steps using a single optimizer and a learning rate, and the runs with the same optimizer with different learning rates are connected by a line, indicating a sweep.
Figure~\ref{fig:main_results}c shows the full root locus plots for 10 optimizers with marked regimes of low-LR, mid-LR, and edge.
Figures~\ref{fig:main_results}d and \ref{fig:main_results}e plot auxiliary smoothness measures $c_t$ and $\rho_t$ that help distinguish the three regimes.
Likewise in Figures~\ref{fig:main_results}a and \ref{fig:main_results}b, sharpness and final loss with respect to the effective lr $h = \eta Q(1)$ are plotted for CNN and ViT models in Figure~\ref{fig:geometry_response_of_cnn_and_vit}.

\begin{figure}[t]
\centering
\begin{minipage}{0.4975\linewidth}
    \vspace{0pt}
    \begin{subfigure}[t]{\textwidth}
        \centering
        \includegraphics[width=\linewidth]{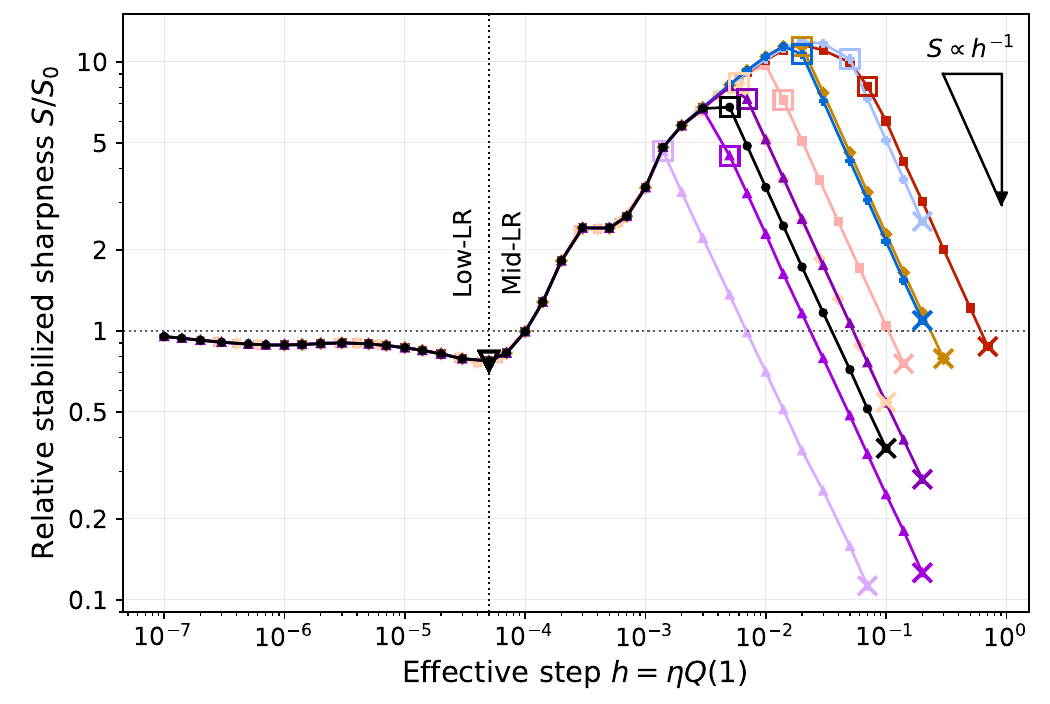}
        \vspace{-1.5em}
        \caption{\small CNN stabilized sharpness $S$ vs. effective lr $h$.}
        \vspace{3pt}
    \end{subfigure}
    \begin{subfigure}[t]{\textwidth}
        \centering
        \includegraphics[width=\linewidth]{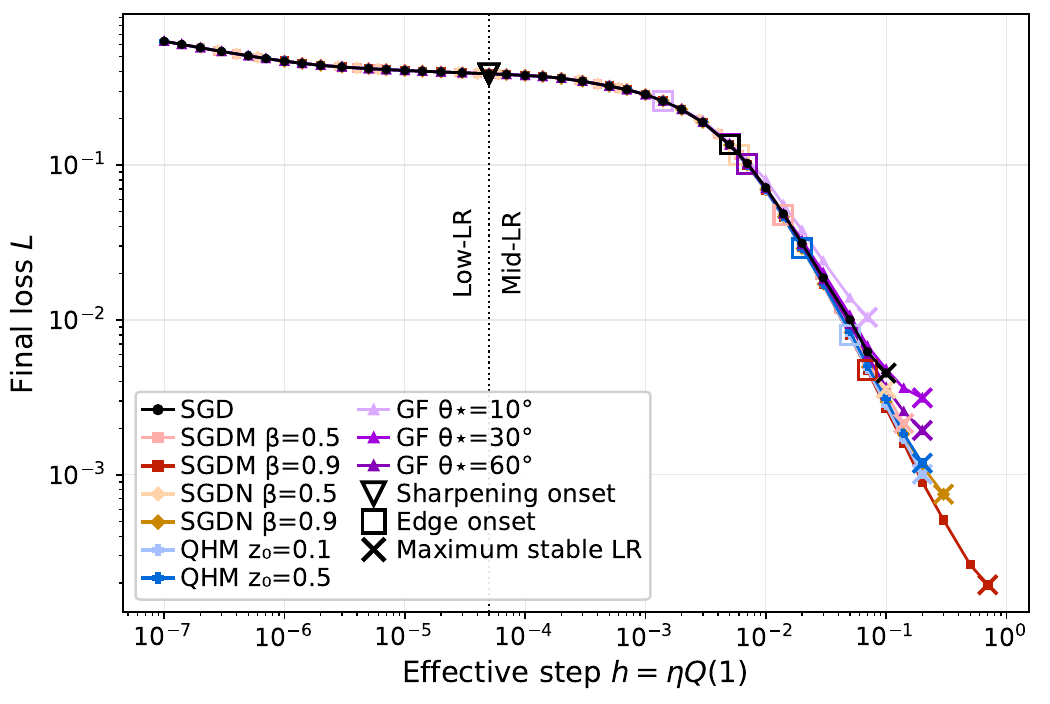}
        \vspace{-1.5em}
        \caption{\small CNN final loss $L$ vs. effective lr $h$.}
    \end{subfigure}
\end{minipage}%
\hfill
\begin{minipage}{0.4975\linewidth}
    \vspace{0pt}
    \begin{subfigure}[t]{\textwidth}
        \centering
        \includegraphics[width=\linewidth]{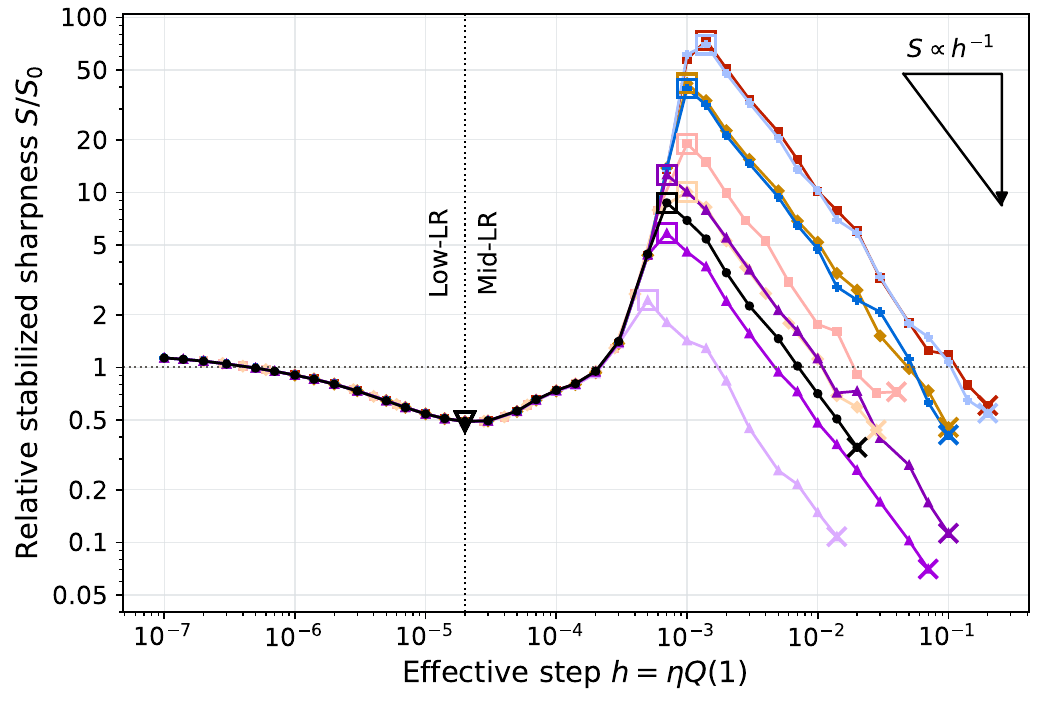}
        \vspace{-1.5em}
        \caption{\small ViT stabilized sharpness $S$ vs. effective lr $h$.}
        \vspace{3pt}
    \end{subfigure}
    \begin{subfigure}[t]{\textwidth}
        \centering
        \includegraphics[width=\linewidth]{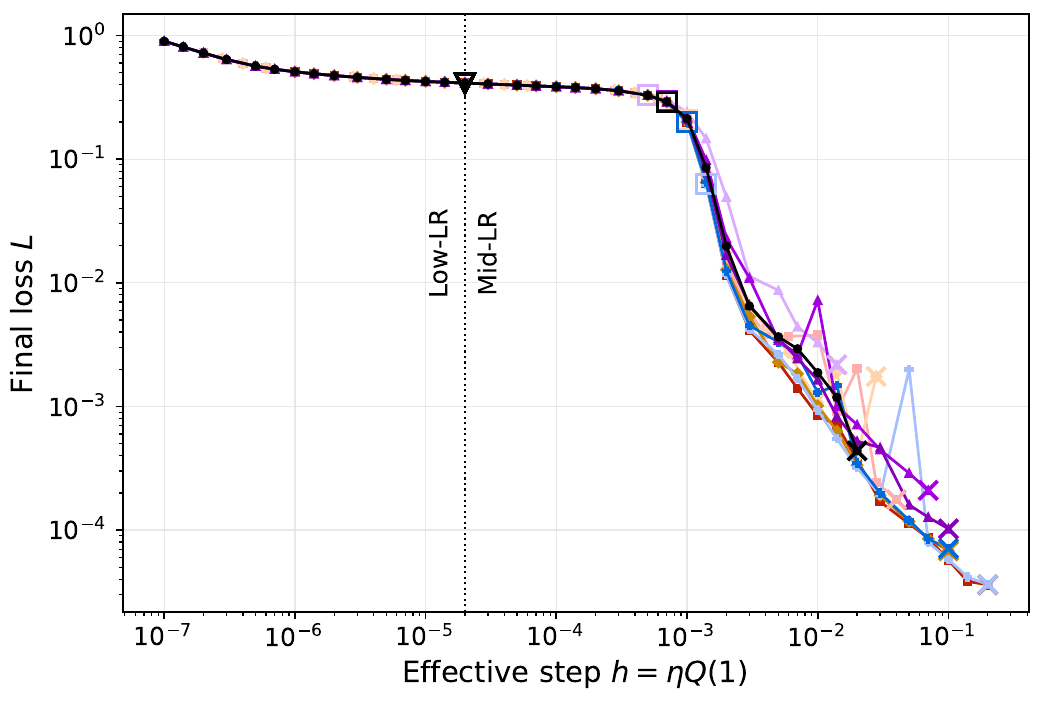}
        \vspace{-1.5em}
        \caption{\small ViT final loss $L$ vs. effective lr $h$.}
    \end{subfigure}
\end{minipage}
\vspace{-.3em}
\caption{%
\small
\textbf{Geometry repsonses of CNN and ViT.}
We observe common structures of low-LR softening $\mathrm dS/\mathrm dh < 0$ and mid-LR sharpening $\mathrm dS/\mathrm dh > 0$ in the sharpness $S$ curve.
However, the curve shapes are unique to each architecture, suggesting that this common curves are the model intrinsic \emph{material properties} independent of optimization algorithms.
}
\label{fig:geometry_response_of_cnn_and_vit}
\end{figure}

Three key observations can instantly be made from the figures.
First, to our surprise, the results from different optimizer families give strikingly similar shapes, especially in the wide range of lower learning rates.
If we align the results with respect to the effective leanring rate $h$ instead of raw $\eta$, the plots almost perfectly overlap, except for the final reciprocal slope of the edge of stability~\citep{cohen2021gradient}.

Second, the geometric responses $(L, S)$ naturally partitions the lr sweep results into three distinct regimes with respect to the effective lr $h = \eta Q(1)$.
(1) In the \emph{low-LR softening regime}, the sharpness drops $\mathrm dS / \mathrm dh < 0$, while the loss $L$ decreases only slightly.
All optimizers give almost identical curves for this regime.
(2) In the \emph{mid-LR sharpening regime}, the sharpness starts to grow $\mathrm dS / \mathrm dh > 0$, while the loss drop becomes amplified.
In this regime, the root loci traverse from their initial points to eventually escape the unit circle.
As a sweep trajectory approaches near \emph{its own} edge of stability, the consecutive gradients alignment $c_t$ starts to break down rapidly.
At the same time, the first order alias $\rho_t$ becomes significant $\rho_t > 1\%$, and thus the discrete-time descent algorithm can no longer be accurately approximated using the continuous-time flow model.
(3) In the final \emph{high-LR edge regime}, the sharpness $S$ becomes strictly reciprocal to the effective lr $S \propto 1/h$, following the edge of stability rule (\ref{eq:edge_of_stability_lee}).
Figure~\ref{fig:stacked_sgd_curves} separates the SGD case to visualize this transition more clearly.

Third, comparing the sharpness curves of MLP (Figure~\ref{fig:main_results}a), CNN (Figure~\ref{fig:geometry_response_of_cnn_and_vit}a), and ViT (Figure~\ref{fig:geometry_response_of_cnn_and_vit}b), we can see that the shapes of the optimizer-independent sharpness $S$ curves in the low- and mid-LR regimes are unique to each architecture.
This shape is thus a model-intrinsic \emph{material property}.
We use the term \emph{material property} to emphasize a property that is predetermined by the problem geometry and independent of the optimization algorithm.
We can additionally observe that all three representative architectures share the same \emph{three regime structures}: the sharpness initially softens in the low-LR regime, sharpens in the mid-LR regime to finally reaches the optimizer-specific edge of stability in the high-LR regime.
In the following sections, we study each of the three regimes in detail, with a special focus on the MLP results in Figure~\ref{fig:main_results}.

\section{Regime I: Low-LR Softening}
\label{sec:observation}
\begin{figure}[t]
\centering
\begin{minipage}{0.331\linewidth}
    \begin{subfigure}[t]{\textwidth}
        \centering
        \includegraphics[width=\linewidth]{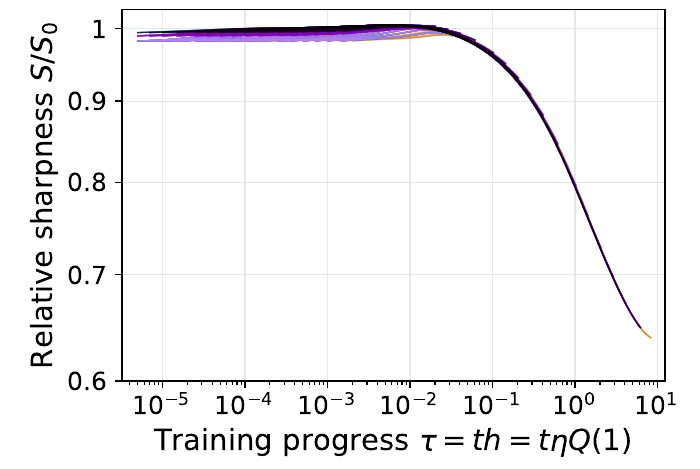}
        \caption{\small Low-LR $\tau$ vs. $S$.}
        \vspace{3pt}
    \end{subfigure}
\end{minipage}%
\hfill
\begin{minipage}{0.331\linewidth}
\hspace{0.01em}
    \begin{subfigure}[t]{\textwidth}
        \centering
        \includegraphics[width=\linewidth]{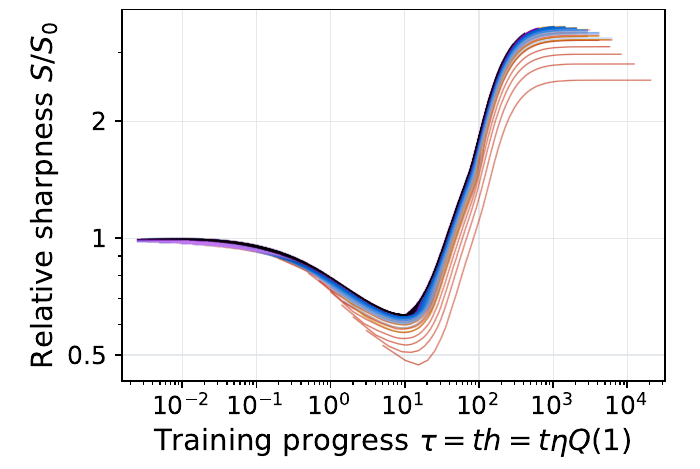}
        \caption{\small Mid-LR $\tau$ vs. $S$.}
        \vspace{3pt}
    \end{subfigure}
\end{minipage}
\hfill
\begin{minipage}{0.331\linewidth}
\hspace{0.01em}
    \begin{subfigure}[t]{\textwidth}
        \centering
        \includegraphics[width=\linewidth]{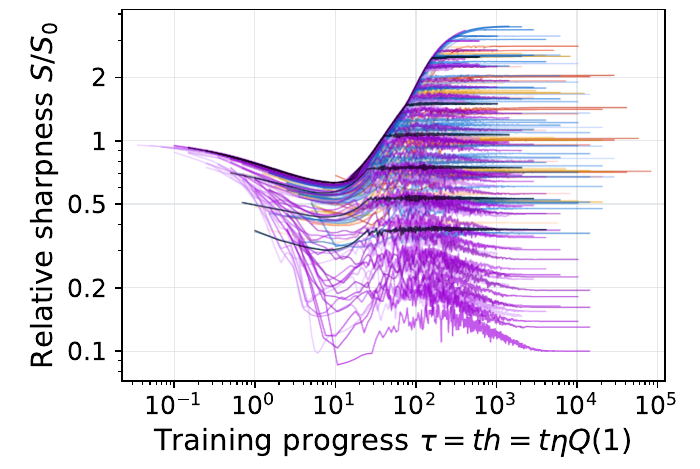}
        \caption{\small High-LR $\tau$ vs. $S$.}
        \vspace{3pt}
    \end{subfigure}
\end{minipage}
\begin{minipage}{0.331\linewidth}
    \begin{subfigure}[t]{\textwidth}
        \centering
        \includegraphics[width=\linewidth]{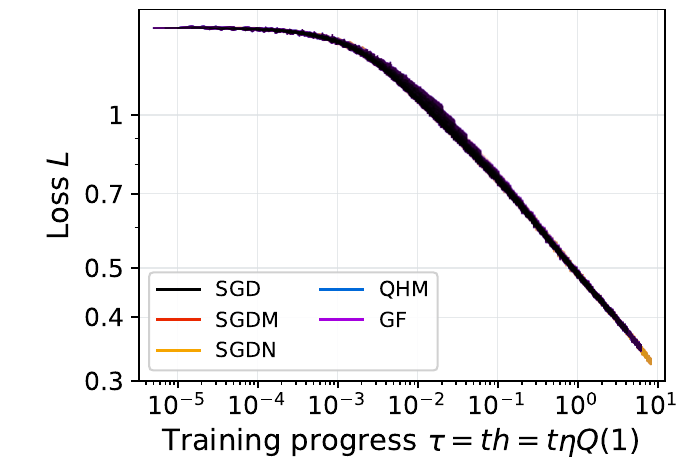}
        \caption{\small Low-LR $\tau$ vs. $L$.}
    \end{subfigure}
\end{minipage}%
\hfill
\begin{minipage}{0.331\linewidth}
\hspace{0.01em}
    \begin{subfigure}[t]{\textwidth}
        \centering
        \includegraphics[width=\linewidth]{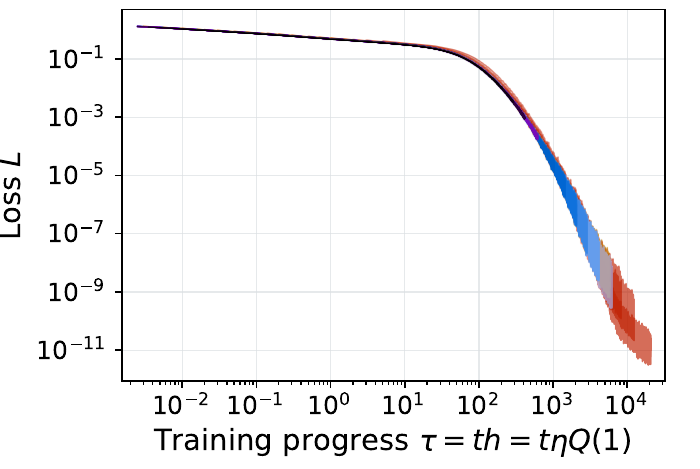}
        \caption{\small Mid-LR $\tau$ vs. $L$.}
    \end{subfigure}
\end{minipage}
\hfill
\begin{minipage}{0.331\linewidth}
\hspace{0.01em}
    \begin{subfigure}[t]{\textwidth}
        \centering
        \includegraphics[width=\linewidth]{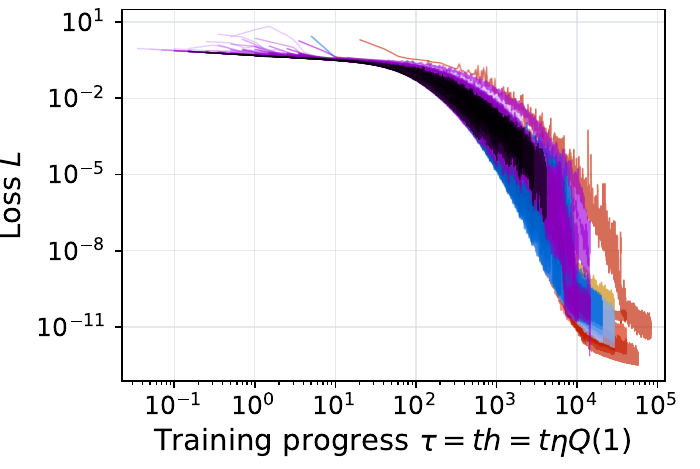}
        \caption{\small High-LR $\tau$ vs. $L$.}
    \end{subfigure}
\end{minipage}
\vspace{-.3em}
\caption{%
\small
\textbf{Geometry responses of MLP with full training traces.}
Instead of picking only the final values as in Figures~\ref{fig:main_results}a and \ref{fig:main_results}b, we sample intermediate values at every 50 iterations for each of the 338 non-blow-up runs, collecting a total of 1.352M entries.
We plot each run as a connected line drawn in the same plot, aligned in the training progress $\tau = ht = \eta Q(1) t$ coordinate.
The low-LR traces almost perfectly overlap, whereas the high-LR traces show clearly different trajectories.
}
\label{fig:full_training_trace_of_mlp}
\end{figure}

The first of the three regimes, the low-LR regime, is characterized by its (1) overlapping traces in the effective lr $h$ and (2) in the training progress $\tau$ coordinates, (3) sharpness softening $\mathrm dS / \mathrm dh < 0$, and (4) shallow loss drop.

\paragraph{Overlapping traces in $h$ coordinate.}
The first thing we notice in Figures~\ref{fig:main_results}a and \ref{fig:main_results}b is that the plots from different optimizers become almost indistinguishable when we align the curves using the effective learning rate $h$ coordinate.
The difference between lr $\eta$ and effective lr $h$ is the dc gain $Q(1)$ of the optimizer.
Normalized optimizers such as simple SGD or normalized Heavy Ball $\vmu_{t+1} = \beta \vmu_t + (1 - \beta) \vg_t$ have a unit dc gain $Q(1) = 1$, hence $h = \eta$.
Conventionally, however, HB or NAG momenta are defined as unnormalized forms: $\vmu_{t+1} = \beta \vmu_t + \vg_t$, yielding nontrivial dc gains $Q(1) = 1/(1-\beta)$.
Normalization separates this secret multiplier $Q(1)$ and merges it into the learning rate $\eta$.
This separation clearly distinguishes the role of the optimizer from that of the lr scheduler: optimizers reshape the parameter updates, while learning rates scale them.
Then the plot coincidence confirms that the sharpness drop in this low-LR regime is independent of the optimizer's temporal structure, hinting a material property of the model-loss complex.

\paragraph{Overlapping traces in $\tau$ coordinate.}
In Figure~\ref{fig:full_training_trace_of_mlp}, we trace the geometric responses $(L, S)$ \emph{per run} by connecting their intermediate values from each run as a line onto the same plot.
It is equally interesting to see that these traces also overlap in the training progress $\tau = \eta Q(1) t$ coordinate.
This graph coincidence means that the training dynamics is insensitive to the optimizer type in this regime, \emph{effectively following the continuous-time gradient flow} $\mathrm d\vtheta / \mathrm d\tau = -\vg$.
Optimizers in this regime only affect through its dc gain $Q(1)$ scaling the timescale $\tau$.
Additional evidences such as almost perfectly aligned gradients $c_t \approx 1$ in Figure~\ref{fig:main_results}d, negligible higher order aliasing $\rho_t \ll 1\%$ in Figure~\ref{fig:main_results}e, and root loci parked in their initial positions in Figure~\ref{fig:main_results}c further confirm this.

\paragraph{Sharpness softening.}
As Figures~\ref{fig:main_results}a, \ref{fig:geometry_response_of_cnn_and_vit}a and \ref{fig:geometry_response_of_cnn_and_vit}b show, the sharpness $S$ decreases from its initial value $S_0 = \lambda_\text{max}(\mH(\vtheta_0))$ until it reaches some minimum value $S_\text{min}$ at $h_\text{min}$.
This \emph{softening} at low effective lr $h$ appears to be a common feature across different architectures, distinguishing the low-LR regime from other regimes with higher effective lr $h$, characterizing the low-LR regime as a \emph{sharpness softening} regime.
Overlapping $S(\tau)$ traces in Figure~\ref{fig:full_training_trace_of_mlp} further identify the sharpness drop $F := \mathrm dS / \mathrm d\tau < 0$ as a material property in this regime.
We call $F$ the \emph{sharpening speed}.

\paragraph{Shallow loss drop.}
For completeness, we inspect the loss $L$ vs $h$ plot in Figure~\ref{fig:main_results}b and Figures~\ref{fig:geometry_response_of_cnn_and_vit}b and \ref{fig:geometry_response_of_cnn_and_vit}d.
Consistently, in all the three curves, the loss does not change much from the initial value until the sharpness turns around at $h_\text{min}$ to start increasing.

\section{Regime II: Mid-LR Sharpening}
\label{sec:theory}
\begin{figure}[t]
\centering
\begin{minipage}{0.248\linewidth}
    \begin{subfigure}[t]{\textwidth}
        \centering
        \includegraphics[width=\linewidth]{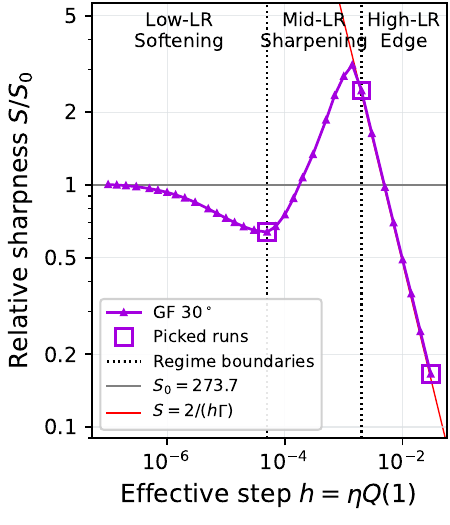}
        \caption{\small $h$ vs. $S$ for GF30.}
    \end{subfigure}
\end{minipage}%
\hfill
\begin{minipage}{0.248\linewidth}
\hspace{0.01em}
    \begin{subfigure}[t]{\textwidth}
        \centering
        \includegraphics[width=\linewidth]{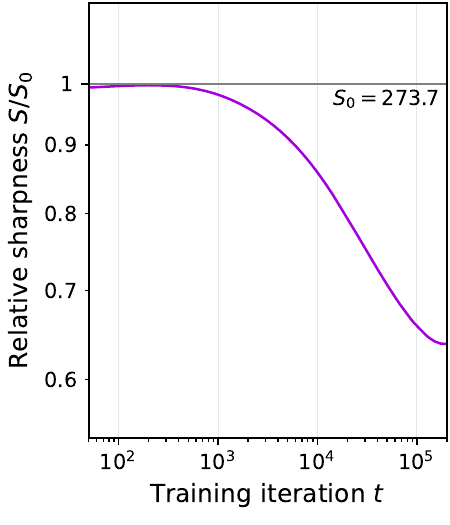}
        \caption{\small Low-LR $t$ vs. $S$.}
    \end{subfigure}
\end{minipage}%
\hfill
\begin{minipage}{0.248\linewidth}
\hspace{0.01em}
    \begin{subfigure}[t]{\textwidth}
        \centering
        \includegraphics[width=\linewidth]{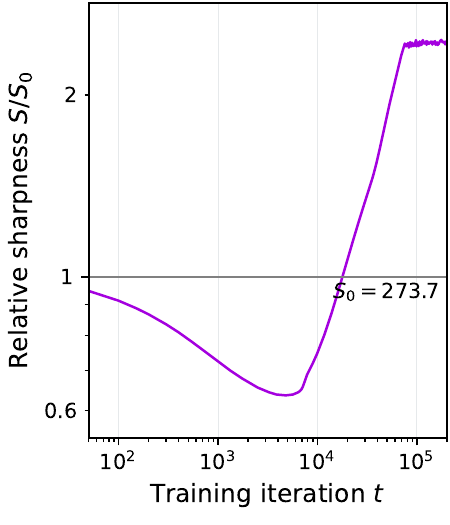}
        \caption{\small Mid-LR $t$ vs. $S$.}
    \end{subfigure}
\end{minipage}
\hfill
\begin{minipage}{0.248\linewidth}
\hspace{0.01em}
    \begin{subfigure}[t]{\textwidth}
        \centering
        \includegraphics[width=\linewidth]{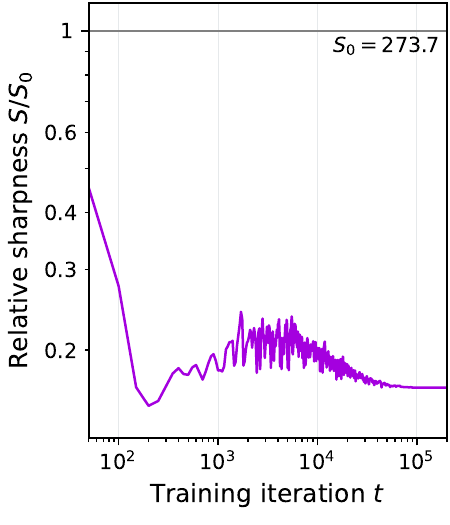}
        \caption{\small High-LR $t$ vs. $S$.}
    \end{subfigure}
\end{minipage}
\vspace{-.3em}
\caption{%
\small
\textbf{Progress sharpening occurrence depends on the effective learning rate $h$.}
We plot traces of the sharpness $S$ over the \emph{true training iterations} $t$ (not $\tau$ this time!) using a single exemplar case of a 30-degree escape Grokfast optimizer (GF30).
Progress sharpening only occurs when the stabilized sharpness is larger than the initial sharpness $S/S_0 > 1$, condition of which is governed by the effective learning rate $h$.
}
\label{fig:progressive_sharpening}
\end{figure}

Past the minimum stabilized sharpness $S(h_\text{min}) = S_\text{min}$, the training dynamics enters the mid-LR regime, which possesses several distinct characteristics than the low-LR regime.
This regime is mainly characterized by (1) \emph{progressive sharpening} of $S$~\citep{cohen2021gradient} and (2) rapid loss drop $L$.
Also, the mid-LR regime is a \emph{transient} regime: it bridges between the material intrinsic continuous-time gradient flow in the low-LR regime and the optimizer-dependent edge of stability in the high-LR regime.
We see this (3) emergence of optimizer-dependence in multiple signatures.

\paragraph{Progressive sharpening.}
The most salient feature of the mid-LR regime is the sharpening.
This can be observed in two different perspectives:
First, the final stabilized sharpness $S(h)$ at the end of training with respect to the effective step size $h$ becomes higher than the initial sharpness $S(h) > S_0$, as Figures~\ref{fig:main_results} and \ref{fig:geometry_response_of_cnn_and_vit} show.
Second, in a single training run, sharpening happens after the sharpness hits the minimum stabilized sharpness $S_\text{min}$ as shown in Figure~\ref{fig:full_training_trace_of_mlp}b and in Figure~\ref{fig:progressive_sharpening}c.
Then the effect manifests as a \emph{progressive sharpening} $F = \mathrm dS / \mathrm d\tau > 0$: the model sharpens itself along the gradient flow.
We can make several observations about this effect:
First, the sharpening first happens before the edge regime, where the geometry responses are material properties, rather than optimizer-dependent.
Second, the sharpening does not always take place, as the previous works often presume~\citep{cohen2021gradient,damian2023self}.
Rather, it is only observed when the final stabilized sharpness $S(h)$ is higher than the initial sharpness $S_0$.
Such condition is determined by the optimizer filter $Q$ and the effective lr $h$.
For example, in Figure~\ref{fig:progressive_sharpening}d, we do not observe progressive sharpening when the learning rate is high enough to make the sharpness stabilized below the initial value $S_0$.
Third, it is shown in Figure~\ref{fig:full_training_trace_of_mlp}c that there exist a certain \emph{material intrinsic upper envelope of the sharpness curve} $S(\tau)$ with respect to the training progress $\tau = \eta Q(1) t$ that upper-bounds the sharpness evolution.
This envelope is dependent on the model, but is independent of the optimizer filter $Q$.
Overall, the progressive sharpening effect is a \emph{material property} of the loss landscape, but it is only realized when the optimizer and the learning rate \emph{jointly approve} the sharpness $S(h)$ to be stabilized at higher values than the initial sharpness $S_0$.

\paragraph{Accelerated loss drop.}
The loss drop curve with respect to the training progress time $\tau = \eta Q(1) t$ in Figure~\ref{fig:full_training_trace_of_mlp}e shows that the loss drop begins to speed up.
This can also be observed in Figures~\ref{fig:main_results}b, \ref{fig:geometry_response_of_cnn_and_vit}b, and \ref{fig:geometry_response_of_cnn_and_vit}d, where the final loss traces become steeper and steeper as the effective lr increases.
The trajectories from different optimizers start to deviate from each other, but this deviation is still small in this regime, limiting the role of the optimizer family in this mid-LR regime.

\begin{wrapfigure}{r}{0.48\linewidth}
\centering
\vspace{-0.5em}
\includegraphics[width=\linewidth]{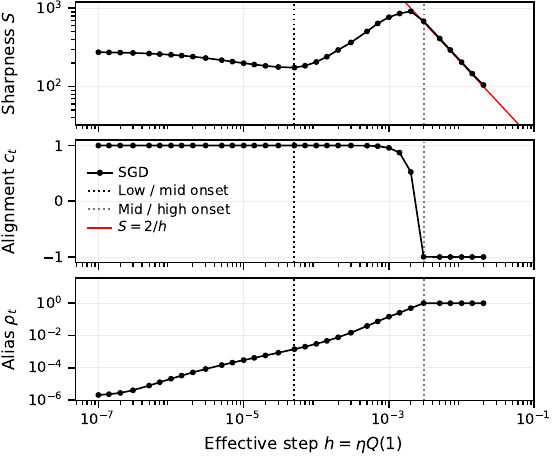}
\vspace{-1.9em}
\caption{%
\small
\textbf{A journey to the edge.}
A subset of Figure~\ref{fig:main_results} showing the SGD curves.
}
\label{fig:stacked_sgd_curves}
\vspace{-2em}
\end{wrapfigure}

\paragraph{Emergence of optimizer-dependence.}
The mid-LR regime connects the optimizer-agnostic low-LR regime and the optimizer-dependent edge regime.
The emergence of optimizer-dependence happens in multiple signatures: (1) aliasing effect $\rho(h)$ consistently increases according to $h$ and becomes significant, (2) consecutive gradients become less aligned, changing its cosine $c_t$ from 1 to the lowest value, and (3) the root locus traverses from the initial pole position to the escape of the unit circle.
Escaping from the unit circle marks the entrance into the edge regime.

We extract the SGD subset of Figure~\ref{fig:stacked_sgd_curves} from Figure~\ref{fig:main_results} to visualize what is happening during this transition more clearly.
In the mid-LR regime, the first order alias $\rho(h)$ increases from less than $1\%$ gradually to $100\%$ (depending on the optimizer family) along the effective lr $h$.
Higher order terms of the Taylor series expansion of the gradient descent step $\mathrm d \vtheta / \mathrm d \tau - \Delta \vtheta / h = O(h)$ become dominant, making the continuous-time gradient flow no longer a good approximation of the training dynamics.
Pure discrete-time effects take place.
The consecutive gradient pairs begin to misalign.
The cosine $c_t$ decreases from 1.
As we prove in Proposition~\ref{prop:second_order_alias}, the first non-zero derivative term of the Taylor series expansion of the cosine $c_t$ is $O(h^2)$, thereby making the alignment $c_t$ as a relatively lagging indicator than the alias $\rho_t$ of this approach to the edge.

At the same time, the optimizer pole traversal in the root loci becomes salient.
In the low-LR regime, the poles remain in the vicinity of the original position.
As the effective lr $h$ and the sharpness $S$ increase, so does the characteristic gain $k = hS$.
The poles traverse across the loci according to the characteristic equation (\ref{eq:eigencomponent_transfer_function_normalized}) and towards the prescribed unit circle escape locations.
Figure~\ref{fig:main_results}c shows how this happens for different optimizers.
It is noteworthy that the escape location on the unit circle only depends on the optimizer shape $\bar Q(z)$ and independent of the learning rate $\eta$, the gain $Q(1)$, the model and other components of the learning system.
We now reach the edge of stability.

\section{Regime III: High-LR Edge}
\label{sec:discussion}
\begin{wrapfigure}{r}{0.34\linewidth}
\centering
\vspace{-1.2em}
\includegraphics[width=\linewidth]{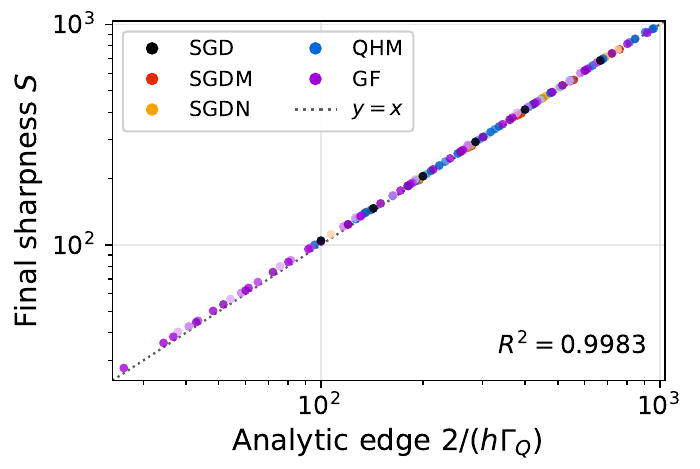}
\vspace{-1.9em}
\caption{%
\small
\textbf{The edge of stability.}
Eq. (\ref{eq:edge_of_stability_lee}) holds for all optimizers and lrs.
}
\label{fig:edge_results}
\vspace{-2em}
\end{wrapfigure}

In the past two regimes, we find that the model geometry changes mostly intrinsically, independently of the optimizer $Q$.
However, we now find that the optimizer $Q$ starts to play a role in determining the training dynamics in this high-LR regime.
Specifically, the unique shape of each optimizer's root locus determines the escape location on the unit circle, the gain at that location determines the position of the edge, and the effective time quantum $h = \eta Q(1)$ determines the sharpness $S$ according to that edge.
We elaborate on these findings in the following.

\begin{figure}[t]
\centering
\begin{minipage}{0.485\linewidth}
    \centering
    \includegraphics[width=\linewidth]{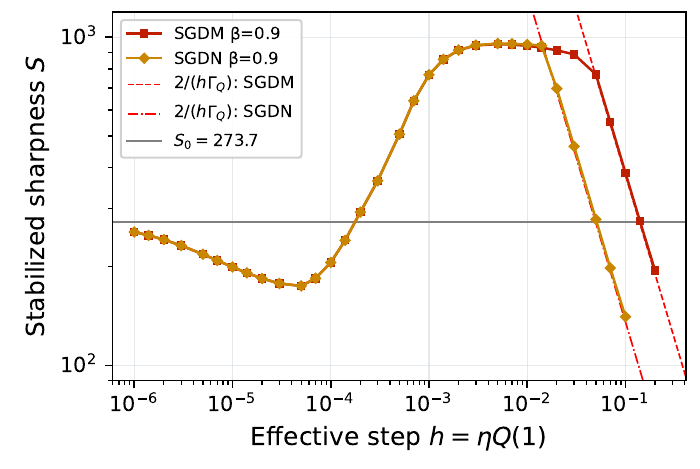}
    \caption{\small 
\textbf{The role of optimizer.}
The final sharpness $S$ of SGD with Heavy Ball and with Nesterov's momentum having the same hyperparameter $\beta = 0.9$.
}
\label{fig:momentum_sharpness}
\end{minipage}%
\hfill
\begin{minipage}{0.485\linewidth}
    \centering
    \includegraphics[width=\linewidth]{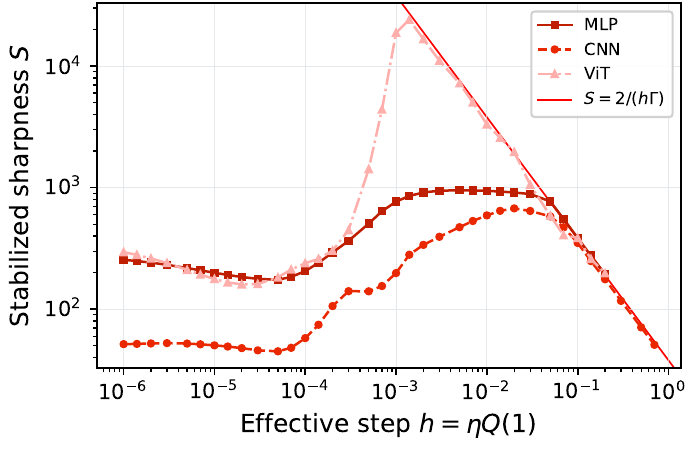}
    \caption{\small 
\textbf{Model chooses the edge entrance, the edge chooses the model sharpness.}
Models sharing an SGDM ($\beta = 0.9$) optimizer share the same edge.
}
\label{fig:architecture_sharpness}
\end{minipage}
\end{figure}

\paragraph{Root locus prescribes the unit circle escape.}
Figure~\ref{fig:main_results}c draws the root loci of the 10 optimizers used throughout this paper, each of which have unique shape of the loci (the curve).
Despite the different shapes, all the loci shares same three-regime behaviors.
The loci start from the initial pole positions ($\times$).
During the low-LR regime, the loci remains relatively at fixed point near the initial point until the minimum sharpness is reached ($\bigtriangledown$).
Then, as mentioned in Section~\ref{sec:theory}, the loci traverses from this point ($\bigtriangledown$) to the edge entrance ($\Box$), which happens at the unit circle escape.
In the high-LR regime, the loci stays at that point ($\Box$) until the training breaks down due to unstably high learning rates.
For one-pole optimizers, this escape point is at the Nyquist frequency ($z = -1$), i.e., the highest frequency indicating a period-2 oscillation.
However, more exotic optimizers with more than one pole, such as Grokfast~\citep{lee2024grokfast}, can escape at other angles.
The gain of the optimizer at this escape point $\Gamma$ is solely determined by the normalized kernel of the optimizer filter $Q(z) / Q(1)$~\citep{lee2026road}:

\begin{equation}
\label{eq:gain_at_escape}
\Gamma(Q) = \max_{\omega \in \Omega_{+}(Q)} \left| \frac{2 Q(e^{i\omega})}{(1 - e^{i\omega}) Q(1)} \right|, \;\; \text{where} \;\;
\Omega_{+}(Q) := \left\{\, \omega \in (0, \pi] : \frac{1 - e^{i\omega}}{Q(e^{i\omega})} \in \mathbb{R}_{>0} \,\right\}.
\end{equation}

\paragraph{Escape gain determines the edge.}
Figure~\ref{fig:edge_results} shows the strict reciprocal relationship (\ref{eq:edge_of_stability_lee}) between the final sharpness $S$ and the effective lr $h$, mediated by the root locus escape gain $\Gamma$ in this high-LR regime.
The stabilized sharpness $S$ can be accurately predicted by the predetermined, optimizer-specific escape gain $\Gamma$, the dc gain $Q(1)$, and the learning rate $\eta$ by $\tilde S = 2 / (\eta Q(1) \Gamma)$.
Figure~\ref{fig:momentum_sharpness} further shows the significance of this escape gain in determining the edge.
SGD with Heavy Ball and with Nesterov's momentum having the same hyperparameter $\beta = 0.9$ share the same dc gain $Q(1) = 1/(1-\beta)$, resulting in almost identical sharpness traces until the edge.
However, their escape gains $\Gamma$ determined by the optimizer root loci are different, resulting in different entrance point into the edge.
Hence the different sharpness only in the high-LR regime.

\paragraph{Effective lr determines the sharpness along the edge.}
Figure~\ref{fig:architecture_sharpness} shows the stabilized sharpness traces of three different models, MLP, CNN, and ViT, trained with the same SGDM ($\beta = 0.9$) optimizer, drawn with respect to the effective learning rate $h$.
Different models have different \emph{material} sharpness responses in the low- and mid-LR regimes, resulting in different entrance points into the edge of stability.
However, since they share the same optimizer, they all share the same edge of stability.
Once the learning system enters the edge regime, the learning rate $h$ takes full charge of determining the sharpness $S$ along the edge, regardless of the model architecture, until the training collapses due to high learning rate.
And this concludes our journey to the edge of stability.

\section{Summary and Conclusion}
\label{sec:conclusion}
We have explored the behavior of a deep learning system under various first order optimization methods.
To our surprise, we found that in the lower learning rates before the edge of stability is reached, the traces of the geometry evolution $(L, S)$ are insensitive to the specific optimization algorithm in use.
The optimizer plays a role in the higher learning rates to steer the learning system into the edge of stability.
Only then, according to the EoS reciprocal rule, the sharpness $S$ becomes predictable from the effective learning rate $h = \eta Q(1)$ until the training goes unstable.

The shared sharpness curve $S(h)$ separates the \emph{material property of the learning system} from the \emph{optimizer-induced edge of stability}.
The evolution of the learning system below the edge of stability is largely prescribed by the model before we define what optimizer to use, but the edge is placed by the optimizer.
The optimizer's role is separated into different mechanistic parts.
The dc gain $Q(1)$ controls the time scale along with the learning rate $\eta$.
Placing the edge of stability is done by the root locus escape gain $\Gamma$ of the remaining normalized filter $Q(z)/Q(1)$, independently of the learning rate and the model being trained.
Therefore, our observations suggest not only to separate the \emph{gain} $\eta Q(1)$ from the \emph{shaping} $Q(z) / Q(1)$ in the description of the first order optimization algorithm, but also to distinguish the problem-specific, \emph{material} properties from the optimizer-induced, \emph{edge of stability}.
It is worth noting that the most effective learning happens at the edge of stability, where the optimizer, the model, the loss, and the data all play their roles in a delicate balance.

\subsection*{AI use statement}

In this work, we used AI assistants solely to assist with writing code and in assisting with mathematical proofs.
All mathematical claims or proofs augmented with AI assistance were thoroughly reviewed, verified, and rewritten as needed by the human authors to ensure correctness and rigor.
We take full responsibility for the final content of this work, including all text, code, and mathematical results.

\subsection*{Reproducibility statement}

We have made every effort to ensure that our results are reproducible.
Sufficient explanations of our methods and setups are provided in Section~\ref{sec:preliminary} and Section~\ref{sec:observation}.
All figures presented in the paper can be fully reproduced using the code and data included in the supplementary material.

\bibliography{main,related_work_new}
\bibliographystyle{iclr2027_conference}

\appendix

\section{Related Work}
\paragraph{Edge of stability.}
Performing a gradient descent according to a quadratic loss function $L(\theta) = \frac{1}{2} \theta^\top \mH \theta$ has a maximum learning rate $\eta_{\max} = 2 / \lambda_{\max}(\mH)$ beyond which the learning starts to diverge.
On the contrary, the edge of stability (EoS)~\citep{cohen2021gradient} states that in a large-scale deep learning, this boundary is often violated without instability.
Instead, it is observed in multiple works~\citep{lyu2022understanding,cohen2023adaptive,damian2023self,agarwala2023second,zhu2023understanding,andreyev2024edge,chen2024stability,cohen2025central,islamov2026noneuclidean,andreyev2026momentum,litman2026origin,lee2026road} that the sharpness of the loss landscape $S := \lambda_{\max}(\mH)$ hovers around this value during deep learning practice.
As a relatively new concept, the governing equation of the edge of stability has been proposed in multiple versions, each of which is based on different scenarios.
\begin{itemize}
    \item Original formulation reported for full-batch gradient descent~\citep{cohen2021gradient}:
    \begin{equation}
    \label{eq:original_sharpness}
    S \;:=\; \lambda_{\max} (\nabla_{\vtheta}^2 L(\vtheta)).
    \end{equation}
    \item Preconditioner $\mP$ updates the eigenspectrum defining sharpness~\citep{cohen2023adaptive}: 
    \begin{equation}
    \label{eq:preconditioned_sharpness}
    S \;:=\; \lambda_{\max} (\mP^{-1/2} \nabla_{\vtheta}^2 L(\vtheta) \mP^{-1/2}).
    \end{equation}
    \item Momentum $\beta$ multiplies a factor $\Gamma(\beta)$ on the left side~\citep{cohen2023adaptive}: 
    \begin{equation}
    \label{eq:momentum_sharpness}
    S \; \leq \; \frac{2}{\eta \Gamma(\beta)}.
    \end{equation}
    \item Stochastic regime changes the definition of sharpness $S$ to Batch Sharpness~\citep{andreyev2024edge}, which is a batch-averaged directional curvature~\citep{lee2023ias,mishkin2024directional} of the loss: 
    \begin{equation}
    \label{eq:batch_sharpness}
    S_\text{B} \;:=\; \mathbb{E}_{\gB \sim \gP_b} \left[ \frac{\nabla_{\vtheta} L_\gB(\vtheta)^\top \nabla_{\vtheta}^2 L_\gB(\vtheta) \nabla_{\vtheta} L_\gB(\vtheta)}{\|\nabla_{\vtheta} L_\gB(\vtheta)\|^2} \right] \leq \frac{2}{\eta},
    \end{equation}
    where $\gP_b$ is the distribution of mini-batches $\gB$ and $L_\gB(\vtheta)$ is the loss from a mini-batch $\gB$.
    \item Mini-batch size modulates the momentum-induced factor $\Gamma(\beta, |\gB|)$~\citep{andreyev2026momentum}:
    \begin{equation}
    \label{eq:mini_batch_sharpness}
    S_\text{B} \; \leq \; \frac{2}{\eta \Gamma(\beta, |\gB|)}.
    \end{equation}
\end{itemize}
Overall, the governing inequality has continuously been evolving based on empirical studies.
In this paper, we fix the batch size in the large batch regime specified by \citet{andreyev2026momentum} and focus on the unpreconditioned maximum Hessian eigenvalue $S = \lambda_{\max}(\mH)$ as the sharpness of the loss landscape, and empirically show that the edge of stability with respect to this sharpness is governed by equation~(\ref{eq:edge_of_stability_lee}) with a correction term $\Gamma(Q)$ that is the root locus escape gain of the optimizer filter $Q$~\citep{lee2026road}.

\paragraph{Driving mechanism of the edge of stability and progressive sharpening.}
Multiple works have been searching for the mechanistic insight into the edge of stability.
Related phenomena of progressive sharpening~\citep{cohen2021gradient,damian2023self}, training instabilities~\citep{gilmer2022loss}, and structured dynamics beyond the quadratic stability boundary are reported to be connected to the edge of stability \citep{jastrzebski2019relation,wang2022analyzing,ahn2022unstable,arora2022understanding}.
To understand the driving mechanism, self-stabilization by higher-order geometry, simplified regression and network models, period-2 oscillation and two-step stability, bifurcation dynamics, and continuous reductions of edge motion \citep{damian2023self,agarwala2023second,zhu2023understanding,chen2023beyond,song2023trajectory,cohen2025central,hofmann2026map} have been proposed.
Phase-diagram analyses show that the leanring rate, the width of the model, the initialization, and the warmpup algorithm participate in sharpness reduction, catapult~\citep{lewkowycz2020large} mechanism, progressive sharpening, and edge of stability into reproducible regimes~\citep{kalra2023phase,kalra2024warmup,kalra2025universal}.
Our contribution is complementary to these works.
We study how the model geometry evolves across diverse set of optimizers and learning rates and uncover a three-regime behavior starting from a shared pre-edge material response to traversal mechanism in the mid-LR regime, and finally to an optimizer-specific root locus escape.
Noteably, previous works~\citep{cohen2021gradient,damian2023self} have presumed the existence of progressive sharpening that transports the model onto the edge of stability without explicitly defining the mechanism.
From a comprehensive study, we show that the progressive sharpening is a model's material property that happens only when the stabilized sharpness is greater than the initial sharpness, and when the effective learning rate allows the sharpness to increase.

\paragraph{First order methods and momentum.}
Classical Heavy Ball~\citep{polyak1964some,sutskever2013importance} and Nesterov's Accelerated Gradient~\citep{nesterov1983method,sutskever2013importance} introduce memory into first-order optimization~\citep{robbins1951stochastic}.
Optimizer with states have been studied in the context of control theory~\citep{su2016differential,lessard2016analysis,gitman2019momentum,muehlebach2021optimization,lee2024grokfast,lee2026road}.
In this context more advanced form of state is proposed, such as quasi-hyperbolic momentum~\citep{ma2018quasi}, AggMo~\citep{lucas2019aggregated}, Grokfast~\citep{lee2024grokfast}, and AdEMAMix~\citep{pagliardini2025ademamix}, that extend the conventional momentum-based optimizers by allowing more than one pole or zeros.
We also treat an unpreconditioned first-order method as a causal gradient filter $Q$ and study its transfer function.
Our observation further suggests to separate its scalar dc gain $Q(1)$ that is absorbed into the effective learning rate $h=\eta Q(1)$ from the normalized temporal shape $\bar Q(z)=Q(z)/Q(1)$.

\paragraph{Learning-rate sweeps and finite-step theory.}
From the decades of practice in deep learning, learning rate has been established as a primary control variable in deep learning \citep{smith2017cyclical}, and scaling-rule and noise-scale studies linked learning rate to batch size and stochastic dynamics \citep{smith2018dont,smith2018bayesian}.
It has been shown that the learning rate can control more than just a time scale of the optimizaiton process.
Large initial learning rates can alter the order of feature acquisition and implicit bias \citep{li2019initial}, and the maximal usable learning rate can differ between initialization and later training \citep{iyer2023maximal}.
Distinction between the discrete-time gradient descent and the continuous-time gradient flow has been formalized by backward-error and modified-equation analyses \citep{li2017stochasticmodified,li2019stochasticmodified,barrett2021implicit,smith2021origin,miyagawa2022equation}.
Recent work also shows that the landscape statistics and transferable learning rates can exhibit architecture-dependent yet highly structured scaling \citep{li2020intrinsic,noci2024super}.
Most of the works focus on the learning rate selection and schedule design, whereas we use a dense fixed-task sweep to study the governing dynamics of the \emph{learning system}.
We find that the dc-adjusted step $h = \eta Q(1)$ indexes the common loss and sharpness responses in the pre-edge regimes, whereas the filter shape $\bar Q(z)$ determines the entry point and the trajectory to the edge.

\label{sec:survey}

\section{Some Mathematics behind the Transition to the Edge}
This section discusses the mathematics behind the transition from the pre-edge gradient flow dynamics to the edge of stability.
In specific, we provide Taylor expansions of the filtered gradient descent that connects the various measurements, i.e., the sharpness $S$, the loss $L$, the alignment $c_t$, and the alias $\rho_t$, from the main text into the same dynamics.
All expansions below are local small $h$ expansions for a fixed smooth objective and a fixed stable filter at an anchored iteration $t$.

\paragraph{Basic notations.}
Let $L(\theta)$ be a scalar objective function and $Q(z)$ be a causal gradient filter with an impulse response $q_j$ likewise in the main text.
A learning rate $\eta > 0$ is a \emph{sampling period} or a \emph{time scale} that realizes a continuous-time gradient flow into a discrete-time gradient descent.
Define $\vg := \nabla L(\theta)$ and $\mH := \nabla^2 L(\theta)$.
Let $\vg_{\le t}$ be a causal sequence of gradients up to \emph{training iteration index} $t$.
Filtered gradient descent updates the parameter $\vtheta$ as:
\begin{equation}
\label{eq:filtered_update}
\vtheta_{t+1} \;=\; \vtheta_t - \eta (Q * \vg_{\le t})_t, \qquad (Q * \vg_{\le t})_t \;=\; \sum_{j\ge 0} q_j \vg_{t-j}.
\end{equation}
$Q(1) = q_0 + q_1 + q_2 + \cdots \in (0, \infty)$ is the dc gain of the filter $Q$.
Using the \emph{effective learning rate} $h := \eta Q(1)$, we can rewrite the update into a normalized form as:
\begin{equation}
\label{eq:filtered_update_normalized}
\vtheta_{t+1} \;=\; \vtheta_t - h (\bar Q * \vg_{\le t})_t, \qquad \bar Q(z) \;=\; \frac{Q(z)}{Q(1)},
\end{equation}
where $\bar Q(1) = 1$.
We additionally define a \emph{canonical training progress} $\tau = ht$ that denotes the time index in the continuous-time gradient flow.
Unless otherwise specified, we simply call $\tau$ the \emph{time}.

\paragraph{Flow, descent, and aliasing.}
Consider a continuous-time gradient flow as a reference dynamics
\begin{equation}
\label{eq:gradient_flow}
\frac{\mathrm d\vtheta}{\mathrm d\tau} \;=\; -\vg.
\end{equation}
This corresponds to the discrete-time update with a time step $h$.
We can \emph{sample} from this continuous-time dynamics to obtain an equivalent discrete-time gradient descent algorithm in the parameter space.
However, due to the discretization, we have an \emph{aliasing term} that signifies the difference between the ideal flow and the acutal learning dynamics.
We can evaluate the corresponding discrete-time gradient descent step in the parameter space from the Taylor expansion of the continuous-time dynamics of equation~(\ref{eq:gradient_flow}).
\begin{equation}
\label{eq:gradient_flow_taylor_expansion}
\vtheta(\tau+h) \;=\; \vtheta(\tau) - h \vg + \frac{h^2}{2} \mH \vg + O(h^3).
\end{equation}
Memoryless gradient descent takes $\vtheta_{t+1} \leftarrow \vtheta_t - h \vg_t$ instead of the exact gradient-flow segment of equation~(\ref{eq:gradient_flow_taylor_expansion}), and therefore differs from the exact gradient-flow segment by
\begin{equation}
\label{eq:gradient_flow_taylor_expansion_difference}
\vtheta_{t+1} - \vtheta(\tau+h) \;=\; -\frac{h^2}{2} \mH \vg_t + O(h^3).
\end{equation}
This is the main source of the trajectory bias relative to gradient flow.

\paragraph{Gradient gradient flow.}
We can also think of a \emph{gradient flow of the gradients} that happens along the \emph{gradient flow of the parameters}.
From the chain rule, we have
\begin{equation}
\label{eq:gradient_flow_of_gradients}
\frac{\mathrm d\vg}{\mathrm d\tau} \;=\; \frac{\mathrm d}{\mathrm d\tau} \left( \nabla L(\vtheta) \right) \;=\; \nabla^2 L(\vtheta) \frac{\mathrm d\vtheta}{\mathrm d\tau} \;=\; -\nabla^2 L(\vtheta) \vg \;=\; -\mH \vg.
\end{equation}
In other words, the continuous-time gradient flow not only transfers the parameter, but also moves the gradient itself.
The Hessian $\mH$ is the \emph{rate of change} of the gradient along the first order ODE.
From this, we can derive the \emph{acceleration} of the parameter along the gradient flow as well by
\begin{equation}
\label{eq:gradient_flow_of_gradients_acceleration}
\frac{\mathrm d^2\vtheta}{\mathrm d\tau^2} \;=\; -\frac{\mathrm d \vg}{\mathrm d\tau} \;=\; \mH \vg.
\end{equation}
Note that this is automatically given from equation~(\ref{eq:gradient_flow}) by the chain rule.
Likewise, any order derivative of the parameter along the flow is automatically given.
If we take another derivative from equation~(\ref{eq:gradient_flow_of_gradients}), we introduce the third order term
\begin{equation}
\label{eq:gradient_flow_of_gradients_second_derivative}
\frac{\mathrm d^2\vg}{\mathrm d\tau^2} \;=\; -\frac{\mathrm d\mH}{\mathrm d\tau} \vg - \mH \frac{\mathrm d \vg}{\mathrm d\tau}.
\end{equation}
The Hessian $\mH$ moves along the gradient flow as well, giving us
\begin{equation}
\label{eq:gradient_flow_of_gradients_second_derivative_combined}
\frac{\mathrm d\mH}{\mathrm d\tau} \;=\; \nabla^3 L \left[\frac{\mathrm d\vtheta}{\mathrm d\tau}\right] \;=\; -\nabla^3 L [\vg].
\end{equation}
Then
\begin{equation}
\label{eq:gradient_flow_of_gradients_second_derivative2}
\frac{\mathrm d^2\vg}{\mathrm d\tau^2} \;=\; \mH^2 \vg + \nabla^3 L [\vg, \vg]
\;=\; -\frac{\mathrm d^3\vtheta}{\mathrm d\tau^3}.
\end{equation}
This gives us the Taylor expansion of the gradient gradient flow up to the second order term
\begin{equation}
\label{eq:gradient_flow_of_gradients_taylor_expansion_second_order}
\vg(\tau + h) \;=\; \vg(\tau) - h \mH(\tau) \vg(\tau) + \frac{h^2}{2} \big(\mH^2(\tau) \vg(\tau) + \nabla^3 L (\tau) [\vg(\tau), \vg(\tau)]\big) + O(h^3).
\end{equation}
Plugging $-jh$ instead of $h$ into the above equation, and let $\vg(\tau) = \vg_t$ and $\mH(\tau) = \mH_t$, we get
\begin{equation}
\label{eq:gradient_flow_of_gradients_taylor_expansion_second_order_backward}
\vg(\tau - jh) \;=\; \vg_t + j h \mH_t \vg_t + \frac{j^2 h^2}{2} \big(\mH^2_t \vg_t + \nabla^3 L_t [\vg_t, \vg_t]\big) + O(h^3).
\end{equation}

\paragraph{Aliasing in gradient gradient flow.}
Taking a simple gradient descent $\vtheta_{t+1} \leftarrow \vtheta_t - h \vg_t$, we get
\begin{equation}
\label{eq:gradient_descent_update_gradient_gradient_flow}
\vg_{t+1} \;=\; \nabla L(\vtheta_{t+1}) \;=\; \nabla L(\vtheta_t - h \vg_t) \;=\; \vg_t - h \mH_t \vg_t + \frac{h^2}{2} \nabla^3 L_t [\vg_t, \vg_t] + O(h^3),
\end{equation}
directly from the Taylor expansion of $L(\vtheta)$ around $\vtheta_t$.
Meanwhile, the Taylor expansion of the \emph{gradient gradient flow} in equation~(\ref{eq:gradient_flow_of_gradients_taylor_expansion_second_order}) just gave us
\begin{equation}
\label{eq:gradient_flow_of_gradients_taylor_expansion_second_order2}
\vg(\tau + h) \;=\; \vg_t - h \mH_t \vg_t + \frac{h^2}{2} \big({\color{red}\mH^2_t \vg_t} + \nabla^3 L_t [\vg_t, \vg_t]\big) + O(h^3),
\end{equation}
and thus
\begin{equation}
\label{eq:gradient_flow_of_gradients_taylor_expansion_second_order3}
\vg_{t+1} - \vg(\tau + h) \;=\; -\frac{h^2}{2} \mH^2_t \vg_t + O(h^3).
\end{equation}
Therefore, we cannot simply assume $\vg_{t+j} = \vg(\tau + jh)$ for any $j \ge 0$.
Comparing equation~(\ref{eq:gradient_flow_taylor_expansion_difference}) and equation~(\ref{eq:gradient_flow_of_gradients_taylor_expansion_second_order3}), we see that the truncation error in the gradient space has one more Hessian than that of the parameter space.
If we define the parameter update vector $\vu_t := (\vtheta_{t+1} - \vtheta_t)/h$ and define
\begin{equation}
\label{eq:gradient_flow_of_gradients_taylor_expansion_second_order4}
\vu_{t:t+j} \;:=\; \frac{\vtheta_{t+j} - \vtheta_t}{h} \;=\; \sum_{i=t}^{t+j-1} \vu_i,
\end{equation}
then we have the exact Taylor expansion of the gradient update $\vg_{t+j}$ as
\begin{equation}
\label{eq:gradient_flow_of_gradients_taylor_expansion_second_order5}
\vg_{t+j} \;=\; \vg_t - h \mH_t \vu_{t:t+j} + \frac{h^2}{2} \nabla^3 L_t [\vu_{t:t+j}, \vu_{t:t+j}] + O(h^3 \|\vu_{t:t+j}\|^3).
\end{equation}
However, this formulation complicates the analysis.
The problem is that $\vu_{t:t+j}$ is not a simple vector, but an aggregation of the sequence of parameter updates, which incorporates changes in mini-batch, evolution of local geometry along the algorithmic trace, filtering effect of the optimizer, and so on.

\paragraph{Quasi-dc expansion and filter effect.}
Fortunately, we can simplify the analysis by assuming the learning rate is sufficiently low, allowing us to use a quasi-dc expansion
\begin{equation}
\label{eq:quasi_dc_expansion}
\vu_{t+i} \;=\; \vg_t + O(h), \quad \text{and} \quad \vu_{t:t+j} \;=\; j\vg_t + O(h).
\end{equation}
Then
\begin{equation}
\label{eq:quasi_dc_expansion_aliasing_term}
\vg_{t-j} \;=\; \vg_t + jh \mH_t \vg_t + O(h^2).
\end{equation}
If we plug this into the convolution of the (normalized) filter $\bar Q * \vg_{\le t}$, we get
\begin{align}
\label{eq:quasi_dc_expansion_aliasing_term2}
(\bar Q * \vg_{\le t})_t &\;=\; \sum_{j\ge 0} \bar q_j \vg_{t-j} \\
&\;=\; \sum_{j\ge 0} \bar q_j ( \vg_t + jh \mH_t \vg_t) + O(h^2) \\
&\;=\; \Big(\sum_{j\ge 0} \bar q_j\Big) \vg_t + \Big(\sum_{j\ge 1} j \bar q_j\Big) h \mH_t \vg_t + O(h^2) \\
&\;=\; \vg_t - h \frac{Q'(1)}{Q(1)} \mH_t \vg_t + O(h^2).
\end{align}
The last line follows from the fact that $Q'(1) = - \sum_{j\ge 0} j q_j = - Q(1) \sum_{j\ge 0} j \bar q_j$.
Therefore,
\begin{equation}
\label{eq:quasi_dc_expansion_aliasing_term3}
\vu_{t} \;=\; (\bar Q * \vg_{\le t})_t \;=\; \vg_t - h \frac{Q'(1)}{Q(1)} \mH_t \vg_t + O(h^2).
\end{equation}
This deserves a couple of theorems.

\begin{theorem}[Filtered gradient descent.]
\label{thm:quasi_dc_expansion_aliasing_term}
For any causal filter $Q$ with an impulse response $q_j$, the filtered gradient descent updates the parameter by
\begin{equation}
\label{eq:quasi_dc_expansion_aliasing_term4}
\vtheta_{t+1} \;\leftarrow\; \vtheta_t - h\vg_t + h^2 \frac{Q'(1)}{Q(1)} \mH_t \vg_t + O(h^3).
\end{equation}
\end{theorem}
\begin{proof}
From $\vtheta_{t+1} \leftarrow \vtheta_t - h\vu_t$ and equation~(\ref{eq:quasi_dc_expansion_aliasing_term3}), we get the result.
\end{proof}

\begin{theorem}[Augmented gradient flow.]
\label{thm:quasi_dc_expansion_aliasing_term_gradient}
Given a causal filter $Q$ with an impulse response $q_j$, the filtered gradient descent corresponds to a continuous-time gradient flow of the modified differential equation
\begin{equation}\tag{{\color{red}$\heartsuit$}}
\label{eq:modified_ode}
\renewcommand{\boxed}[1]{\colorbox{lightgray!15}{\ensuremath{#1}}}
\boxed{\ %
\displaystyle
\frac{\mathrm d\vtheta}{\mathrm d\tau} \;=\; -\vg - h \alpha_Q \mH \vg,
}
\end{equation}
where
\begin{equation}
\label{eq:alpha_Q}
\alpha_Q \;:=\; \frac{1}{2} - \frac{Q'(1)}{Q(1)},
\end{equation}
up to $O(h^2)$.
\end{theorem}
\begin{proof}
From the Taylor expansion of equation~(\ref{eq:modified_ode}) to $O(h^2)$, we get
\begin{equation}
\label{eq:modified_ode_taylor_expansion}
\vtheta_{t+1} = \vtheta_t - h\vg_t - h^2 \alpha_Q \mH_t \vg_t + \frac{h^2}{2} \mH_t \vg_t + O(h^3),
\end{equation}
where the last term comes from expanding $\vg(\vtheta_t - h\vg_t) \approx \vg_t - h\mH_t\vg_t$.

Combining like $h^2$ terms, the update becomes
\begin{equation}
\label{eq:modified_ode_taylor_expansion2}
\vtheta_{t+1} = \vtheta_t - h\vg_t + h^2 \left(\frac{1}{2} - \alpha_Q\right) \mH_t \vg_t + O(h^3).
\end{equation}
Comparing this to equation~(\ref{eq:quasi_dc_expansion_aliasing_term4}), we match coefficients and require
\begin{equation}
\label{eq:modified_ode_taylor_expansion3}
\alpha_Q = \frac{1}{2} - \frac{Q'(1)}{Q(1)}.
\end{equation}
Thus, the filtered gradient descent update is matched by the ODE~(\ref{eq:modified_ode}) up to $O(h^2)$, with $\alpha_Q$ as defined in~(\ref{eq:alpha_Q}).
\end{proof}

We can likewise compare the gradient trajectories from the filtered gradient descent with that of the augmented gradient flow.
By applying Theorem~\ref{thm:quasi_dc_expansion_aliasing_term} $j$ times with a fixed number of steps $j$, we get
\begin{align}
\label{eq:gradient_descent_update_gradient_gradient_flow_j_taylor_expansion}
\vu_{t+i} &\;=\; (\bar Q * \vg_{\le t+i})_{t+i} \\
&\;=\; \sum_{l\ge 0} \bar q_l \vg_{t+i-l} \\
&\;=\; \sum_{l\ge 0} \bar q_l \big(\vg_t + (l-i) h \mH_t \vg_t\big) + O(h^2) \\
&\;=\; \vg_t - h \left(\frac{Q'(1)}{Q(1)} + i\right) \mH_t \vg_t + O(h^2).
\end{align}
Summing over $i$ from $0$ to $j-1$, we get
\begin{equation}
\label{eq:gradient_descent_update_gradient_gradient_flow_j}
\vu_{t:t+j} \;=\; \sum_{i=0}^{j-1} \vu_{t+i} \;=\; j \vg_t - h \left(j\frac{Q'(1)}{Q(1)} + \frac{j(j-1)}{2}\right) \mH_t \vg_t + O(h^2).
\end{equation}
Plugging this into equation~(\ref{eq:gradient_flow_of_gradients_taylor_expansion_second_order5}), we get
\begin{equation}
\label{eq:gradient_descent_update_gradient_gradient_flow_j_taylor_expansion2}
\vg_{t+j} \;=\; \vg_t - jh \mH_t \vg_t + h^2 \left[
\left(j\frac{Q'(1)}{Q(1)} + \frac{j(j-1)}{2}\right) \mH_t^2 \vg_t +
\frac{j^2}{2} \nabla^3 L_t [\vg_t, \vg_t]\right] + O(h^3).
\end{equation}
On the other hand, plugging $jh$ in the exact Taylor expansion of $\vg(\tau)$ gives
\begin{equation}
\label{eq:gradient_descent_update_gradient_gradient_flow_j_taylor_expansion3}
\vg(\tau + jh) \;=\; \vg_t - jh \mH_t \vg_t + \frac{j^2 h^2}{2} \Big(
\mH_t^2 \vg_t + \nabla^3 L_t [\vg_t, \vg_t]\Big) + O(h^3).
\end{equation}
Comparing these two equations, we see that they match up to $O(h^2)$.

\begin{corollary}[Gradient consistency.]
\label{cor:gradient_descent_update_gradient_gradient_flow_j}
Gradients from the continuous-time gradient flow $\mathrm d\vg / \mathrm d\tau = -\mH \vg$ and those from the filtered gradient descent of equation~(\ref{eq:quasi_dc_expansion_aliasing_term4}) agree up to $O(h^2)$ following
\begin{equation}
\label{eq:gradient_descent_update_gradient_gradient_flow_j_corollary}
\vg_{t+j} - \vg(\tau + jh) \;=\; -jh^2 \alpha_Q \mH_t^2 \vg_t + O(h^3).
\end{equation}
\end{corollary}
\begin{proof}
Subtracting equation~(\ref{eq:gradient_descent_update_gradient_gradient_flow_j_taylor_expansion3}) from equation~(\ref{eq:gradient_descent_update_gradient_gradient_flow_j_taylor_expansion2}) obtains the result.
\end{proof}

\paragraph{Local geometric responses in the low-LR regime.}
Once the accurate approximations for the parameter and gradient propagation are established, approximating the local geometric responses is straightforward.
From the chain rule,
\begin{equation}
\label{eq:local_loss_taylor}
\frac{\mathrm dL}{\mathrm d\tau} \;=\; \nabla L^\top \frac{\mathrm d\vtheta}{\mathrm d\tau},
\end{equation}
Let $S(\vtheta)=\lambda_{\max}(\mH(\vtheta))$. Likewise,
\begin{equation}
\label{eq:local_sharpness_taylor}
\frac{\mathrm dS}{\mathrm d\tau} \;=\; \nabla S^\top \frac{\mathrm d\vtheta}{\mathrm d\tau},
\end{equation}
where
\begin{equation}
\label{eq:local_sharpness_taylor2}
\nabla S \;=\; \nabla^3 \!L[\vv_{\max}, \vv_{\max}].
\end{equation}
Here, $\vv_{\max}$ is the maximum eigenvector of the Hessian $\mH$.
From the observations in the main text, we see that in the low-LR regime, the geometric responses are almost identical, implying that the traces follow the same, continuout-time gradient flow in the progress coordinate $\tau = h t$.
Mathematically, this corresponds to $\mathrm d\vtheta / \mathrm d\tau = -\vg$.
\begin{align}
\label{eq:local_loss_taylor_vg}
\frac{\mathrm dL}{\mathrm d\tau} &\;=\; -\nabla L^\top\vg \;=\; -\|\vg\|^2, \\
\label{eq:local_sharpness_taylor_vg}
\frac{\mathrm dS}{\mathrm d\tau} &\;=\; -\nabla S^\top\vg \;=\; -\nabla^3\!L[\vv_{\max}, \vv_{\max}, \vg].
\end{align}

\paragraph{First order alias $\rho_t$.}
The first order alias $\rho_t$ defined in Section~\ref{sec:preliminary} actually was a scaled normalized gradient difference,
\begin{equation}
\label{eq:consecutive_gradient_alignment_rewrite}
\rho_t \;:=\; |\alpha_Q| \frac{\|\vg_{t+1} - \vg_{t}\|}{\|\vg_t\|}.
\end{equation}
We now prove that $\rho_t$ \emph{is} the \emph{first order alias}, the magnitude of the higher order aliasing terms normalized by the norm of the ideal first order gradient flow $\dot \vtheta_t = -\vg_t$.

\begin{proposition}[First order alias.]
\label{prop:first_order_alias}
Assume a parameter update is made by the filtered gradient descent $\vtheta_{t+1} \leftarrow \vtheta_t - \eta (Q * \vg_{\le t})_t$.
By Theorem~\ref{thm:quasi_dc_expansion_aliasing_term_gradient}, a continuous-time gradient flow defined by
\begin{equation}
\label{eq:first_order_alias_corollary}
\dot {\vtheta}_t^{\textnormal{equiv}} \;=\; -\vg_t - h \alpha_Q \mH_t \vg_t,
\end{equation}
matches this discrete-time update up to $O(h^2)$.
Then $\rho_t$ in equation~(\ref{eq:consecutive_gradient_alignment_rewrite}) is equal to the normalized difference between this augmented flow and the ideal gradient flow $\dot \vtheta_t = -\vg$ defined with respect to the progress coordinate $\tau = ht = t \eta Q(1)$, up to $O(h^2)$.
\begin{equation}
\label{eq:first_order_alias_proposition}
\rho_t \;=\; \frac{\|\dot {\vtheta}_t^{\textnormal{equiv}} - \dot \vtheta_t\|}{\|\dot \vtheta_t\|} + O(h^2).
\end{equation}
\end{proposition}
\begin{proof}
From the difference between the consecutive gradients
\begin{equation}
\label{eq:consecutive_gradient_alignment_rewrite1}
\Delta \vg_t \;:=\; \vg_{t+1} - \vg_{t} \;=\; -h \mH_t \vg_t + O(h^2).
\end{equation}
Plugging this into the definition of $\rho_t$, we get
\begin{equation}
\label{eq:consecutive_gradient_alignment_rewrite2}
\rho_t \;=\; |\alpha_Q| \frac{\|{-h \mH_t \vg_t} + O(h^2)\|}{\|\vg_t\|}
\;=\; \frac{\|(-\vg_t - h \alpha_Q \mH_t \vg_t + O(h^2)) - (-\vg_t)\|}{\|{-\vg_t}\|}.
\end{equation}
Substituting the definitions of the augmented flow and the ideal gradient flow returns the result.
\end{proof}

From equation~(\ref{eq:consecutive_gradient_alignment_rewrite2}), we further get
\begin{equation}
\label{eq:consecutive_gradient_alignment_rewrite3}
\rho_t \;=\; h |\alpha_Q| \frac{\|\mH_t \vg_t\|}{\|\vg_t\|} + O(h^2) \;=\; O(h).
\end{equation}
That is, the first order alias $\rho_t$ is a \emph{first-order observable of the aliasing effect} of the equivalent continuous-time gradient flow.
Figure~\ref{fig:main_results}e shows the first order alias $\rho_t$ as a function of the timescale $h = \eta Q(1)$, which is approximately linear in the low-LR and mid-LR regimes, until the learning system starts to approach the edge of stability where the first order approximation breaks down.

\paragraph{Why consecutive gradient alignment $c_t$ is a lagging indicator.}
In Section~\ref{sec:preliminary}, we also defined the cosine of the consecutive gradients as a measure of gradient alignment,
\begin{equation}
\label{eq:ct_cosine}
c_t\;:=\;\frac{\langle \vg_t, \vg_{t+1} \rangle}{\|\vg_t\| \|\vg_{t+1}\|}.
\end{equation}
We now prove that $c_t$ is a \emph{second-order observable of the aliasing effect}.

\begin{proposition}[Second order alias.]
\label{prop:second_order_alias}
Define $c_t$ as in equation~(\ref{eq:ct_cosine}), and $\rho_t$ as in equation~(\ref{eq:consecutive_gradient_alignment_rewrite}).
For a small drift in gradient $\Delta \vg_t = \vg_{t+1} - \vg_t$, we have
\begin{equation}
\label{eq:second_order_alias_proposition}
1 - c_t \;=\; O(\rho_t^2).
\end{equation}
\end{proposition}
\begin{proof}
Define a projection operator $\Pi_{\vg_t}^{\perp}$ as
\begin{equation}
\label{eq:second_order_alias_proof1}
\Pi_{\vg_t}^{\perp} \;=\; I - \frac{\vg_t \vg_t^\top}{\|\vg_t\|^2},
\end{equation}
that normalizes a vector to be perpendicular to $\vg_t$.
Then the cosine $c_t$ can be expanded as
\begin{equation}
\label{eq:second_order_alias_proof2}
c_t \;=\; 1 - \frac{\left\| \Pi_{\vg_t}^{\perp} \Delta \vg_t \right\|^2}{2\|\vg_t\|^2} + O\left(\frac{\|\Delta \vg_t\|^3}{\|\vg_t\|^3}\right).
\end{equation}
This completes the proof.
\end{proof}

Moreover, since $\Pi_{\vg_t}^{\perp}$ is a projection operator, we can define a sine function of the gradient drift as
\begin{equation}
\label{eq:second_order_alias_proof3}
\sin \psi_t \;:=\; \frac{\left\| \Pi_{\vg_t}^{\perp} \Delta \vg_t \right\|}{\|\Delta \vg_t\|}.
\end{equation}
Then $c_t$ can be rewritten as
\begin{equation}
\label{eq:second_order_alias_proof4}
1 - c_t \;=\; \frac{\rho_t^2}{2 \alpha_Q^2}\sin^2 \psi_t + O(\rho_t^3).
\end{equation}
Immediately, we can see that $1 - c_t$ is a second-order observable of the aliasing effect.
\begin{equation}
\label{eq:second_order_alias_proof5}
\rho_t \;=\; O(h), \qquad 1 - c_t \;=\; O(h^2).
\end{equation}
Therefore, the cosine $c_t$ is a lagging indicator of the aliasing effect, compared to the alias $\rho_t$.
In the main text, by comparing Figure~\ref{fig:main_results}d and Figure~\ref{fig:main_results}e, we can observe a clear difference between the two measurements: 
the alias $\rho_t$ gradually increases linearly from insignificant to significant values according to the effective learning rate $h$ throughout the low-LR and mid-LR regimes, while the cosine $c_t$ remains near 1 until the learning system closely approaches the edge of stability.
Only then $c_t$ suddenly breaks down.
The edge of stability emerges as a pure discrete-time phenomenon after the first and the second order aliasing become significant.

\label{sec:proofs}

\section{Full Experimental Results}
We finally provide the full tables and figures for the experiments summarized in the main text.

\subsection{Optimizer Specification}
\label{sec:optimizer_specification}

This section states the hyperparameters and characteristic quantities of the 10 optimizers we used.
First, Table~\ref{tab:filter_parameters} gives the full specification of the optimizers.
In specific, we use one SGD~\citep{robbins1951stochastic}, two SGD with Heavy Ball (SGDM)~\citep{polyak1964some}, two SGD with Nesterov's Acceleration (SGDN)~\citep{nesterov1983method}, two QHM~\citep{ma2018quasi} and three GF~\citep{lee2024grokfast} optimizers.
We vary the momentum hyperparameter $\beta$ for SGDM and SGDN, the zero position $\nu$ for QHM, and the root locus escape angle $\theta$ for GF.
The chosen hyperparameters are denoted in their name suffixes so that we can easily identify them in the tables and figures.
All optimizers are based on stochastic mini-batch gradient descent with a fixed batch size of 4096.
This batch size is chosen to be large enough to be inside the reported \emph{large batch regime} of \citet{andreyev2026momentum} above the critical batch size of CIFAR-10 task~\citep{krizhevsky2009learning}.

Each column of Table~\ref{tab:filter_parameters} is a derived quantity from the optimizer transfer function $Q(z)$ in equation~(\ref{eq:eigencomponent_transfer_function_normalized}): $z - 1 + k \bar Q(z) = 0$ for some gain $k = h S = \eta Q(1) \lambda_{\max}(\mH)$, its first derivative $Q'(z)$, and its root locus.
$Q(1)$ is the dc gain and $Q'(1)$ is the first order Taylor coefficient of $Q(z)$ around the dc point $z=1$.
The quantity $\alpha_Q=\frac12-\frac{Q'(1)}{Q(1)}$ is the second order Taylor coefficient of the filtered gradient descent, used as a scaling factor in evaluating the first order alias $\rho_t$ in Section~\ref{sec:preliminary} in the main text and later used in the proof of Theorem~\ref{thm:quasi_dc_expansion_aliasing_term_gradient} in the previous section.
Define $k_\star$ as the value of coefficient $k$ in the characteristic equation $z - 1 + k \bar Q(z) = 0$ for which the first root reaches the unit circle.
Then the unit circle escape gain $\Gamma_Q$ of the optimizer $Q(z)$ is the reciprocal of $k_\star$: $\Gamma_Q = 2 / k_\star$.
Therefore, the sharpness $S$ at the edge of stability is given by
\begin{equation}
S \;=\; \frac{2}{h \Gamma_Q} \;=\; \frac{2}{\eta Q(1) \Gamma_Q} \;=\; \frac{k_\star}{\eta Q(1)}.
\end{equation}
This prediction is empirically justified in the main text, e.g., in Figure~\ref{fig:main_results} and~\ref{fig:geometry_response_of_cnn_and_vit}.

The quantities $k_\star$ and $\Gamma_Q$ can be analytically derived from the normalized optimizer transfer function $\bar Q(z)$ as follows.
SGD, SGDM, SGDN, and QHM all have one pole.
Their transfer function $\bar Q(z)$ is of the form
\begin{equation}
\bar Q(z) \;=\; \frac{b_0 + b_1 z^{-1}}{1 - \beta z^{-1}}, \qquad k_\star \;=\; \frac{2(1 + \beta)}{b_0 - b_1},
\end{equation}
with its unit circle escape at the Nyquist frequency $z = -1$.
On the other hand, Grokfast with an cascade of two identical filters of momentum coefficient $\beta$ has two poles.
Its normalized transfer function $\bar Q(z)$ is of the form
\begin{equation}
\bar Q(z) \;=\; \frac{c}{(1 - \beta z^{-1})^2}, \qquad \cos \theta_\star = \frac{\beta^2 + 2\beta - 1}{2\beta^2}, \qquad k_\star \;=\; \frac{1 + 2 \beta - \beta^2 - 2 \cos \theta_\star}{c},
\end{equation}
with its unit circle escape at the angle $\theta_\star$.
All these values are summarized in Table~\ref{tab:filter_parameters}.

\subsection{Regime Boundaries}
\label{sec:regime_boundaries}

Table~\ref{tab:regime_boundaries} gives the regime boundaries and the number of runs fall into each category of the regimes for the MLP experiments.
Since all curves coincide in the low-LR and early mid-LR regimes in the effective learning rate $h$ coordinate, all of the 10 optimizers share the same low-LR/mid-LR boundaries $h_\text{low|mid} = 5\times 10^{-5}$.
As we showed in Figure~\ref{fig:main_results}a in the main manuscript, the curves start to deviate from each other in the mid-LR regime, giving different mid-LR/high-LR boundaries $h_\text{mid|high}$ for each optimizer.
Dividing this value by the dc gain $Q(1)$ gives the edge position in the raw learning rate coordinate $\eta_{\text{edge}} = h_{\text{mid|high}} / Q(1)$.
The columns $N_{low}$, $N_{mid}$, and $N_{high}$ are the numbers of learning rate sweep runs fall into each of the regimes.

\subsection{Individual Plots}
\label{sec:individual_plots}

We finally draw each run summarized as a single dot in the plots of the main text into full graphs.

\begin{table}[t]
\caption{\small
\textbf{Characteristic quantities of the first order optimizers used in the experiments.}
The parameters of the filters $Q(z)$ and the corresponding edge gains.
The edge gain $\Gamma_Q$ is twice the reciprocal of the first positive gain $k_\star$ for which a root of $z-1+k\bar Q(z)=0$ reaches the unit circle.
}
\label{tab:filter_parameters}
\centering
\begin{tabular}{l|rrr|rrr}
\toprule
Optimizer & $Q(1)$ & $Q'(1)$ & $\alpha_Q=\frac12-\frac{Q'(1)}{Q(1)}$ & $\theta_\star$ & $k_\star$ & $\Gamma_Q$ \\
\midrule
SGD & 1 & 0 & 0.5 & $180^\circ$ & 2 & 1 \\
SGDM50 & 2 & -2 & 1.5 & $180^\circ$ & 6 & 0.33333333 \\
SGDM90 & 10 & -90 & 9.5 & $180^\circ$ & 38 & 0.052631579 \\
SGDN50 & 2 & -1 & 1 & $180^\circ$ & 3 & 0.66666667 \\
SGDN90 & 10 & -81 & 8.6 & $180^\circ$ & 13.571429 & 0.14736842 \\
QHM01 & 1 & -8.8888889 & 9.3888889 & $180^\circ$ & 31.090909 & 0.064327485 \\
QHM05 & 1 & -8 & 8.5 & $180^\circ$ & 12.666667 & 0.15789474 \\
GF10 & 1 & -11.473713 & 11.973713 & $10^\circ$ & 0.37900746 & 5.2769409 \\
GF30 & 1 & -3.8637033 & 4.3637033 & $30^\circ$ & 1.3032254 & 1.534654 \\
GF60 & 1 & -2 & 2.5 & $60^\circ$ & 3 & 0.66666667 \\
\bottomrule
\end{tabular}
\end{table}

\begin{table}[t]
\caption{\small
\textbf{Regime boundaries of MLP experiments.}
All optimizers have the same low-LR/mid-LR boundaries $h_\text{low|mid}$, since their curves coincide into a single curve.
However, their mid-LR/high-LR boundaries $h_\text{mid|high}$ differ.
We also provide the edge position in the raw learning rate coordinate $\eta_{\text{edge}} = h_{\text{mid|high}} / Q(1)$.
Finally, we also provide the number of runs fall into each category of the regimes $N_{low}$, $N_{mid}$, and $N_{high}$ for each optimizer.
}
\label{tab:regime_boundaries}
\centering
\begin{tabular}{l|rrr|rrr}
\toprule
Optimizer & $h_{\text{low|mid}}$ & $h_{\text{mid|high}}$ & $\eta_{\text{edge}}$ & $N_{low}$ & $N_{mid}$ & $N_{high}$ \\
\midrule
SGD & 5e-05 & 0.003 & 0.003 & 16 & 11 & 6 \\
SGDM50 & 5e-05 & 0.01 & 0.005 & 15 & 13 & 5 \\
SGDM90 & 5e-05 & 0.05 & 0.005 & 10 & 18 & 5 \\
SGDN50 & 5e-05 & 0.004 & 0.002 & 15 & 11 & 6 \\
SGDN90 & 5e-05 & 0.02 & 0.002 & 10 & 16 & 5 \\
QHM01 & 5e-05 & 0.05 & 0.05 & 16 & 18 & 4 \\
QHM05 & 5e-05 & 0.014 & 0.014 & 16 & 15 & 6 \\
GF10 & 5e-05 & 0.0007 & 0.0007 & 16 & 7 & 8 \\
GF30 & 5e-05 & 0.002 & 0.002 & 16 & 10 & 8 \\
GF60 & 5e-05 & 0.005 & 0.005 & 16 & 12 & 8 \\
\bottomrule
\end{tabular}
\end{table}

\begin{figure}[p]
    \centering
    \includegraphics[width=\textwidth]{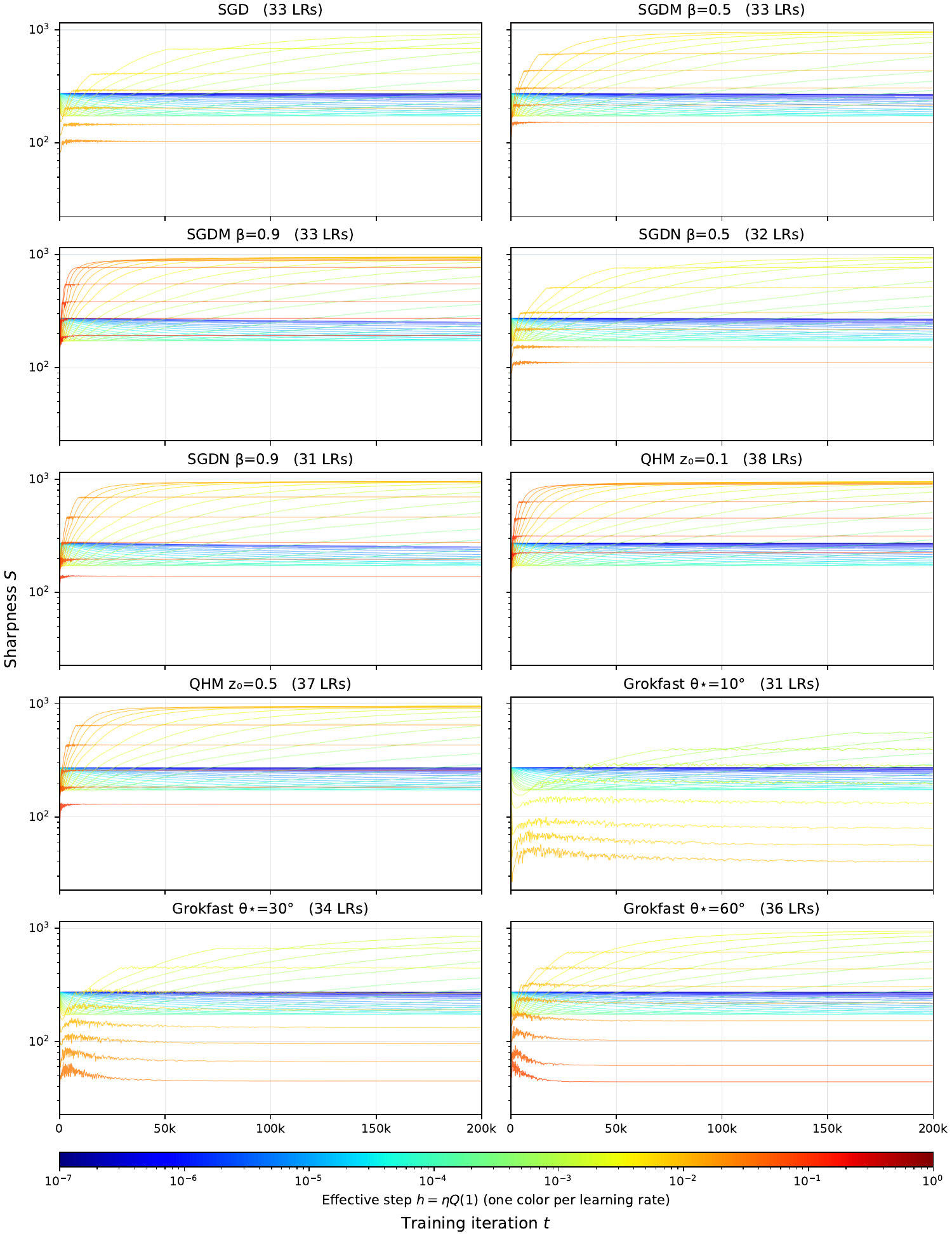}
    \caption{\small \textbf{Individual plots of MLP sharpness $S$.} Plots are drawn with respect to the training iteration $t$ coordinate. Different learning rates are colored in different colors.}
    \label{fig:mlp-iteration-sharpness-grid}
\end{figure}

\begin{figure}[p]
    \centering
    \includegraphics[width=\textwidth]{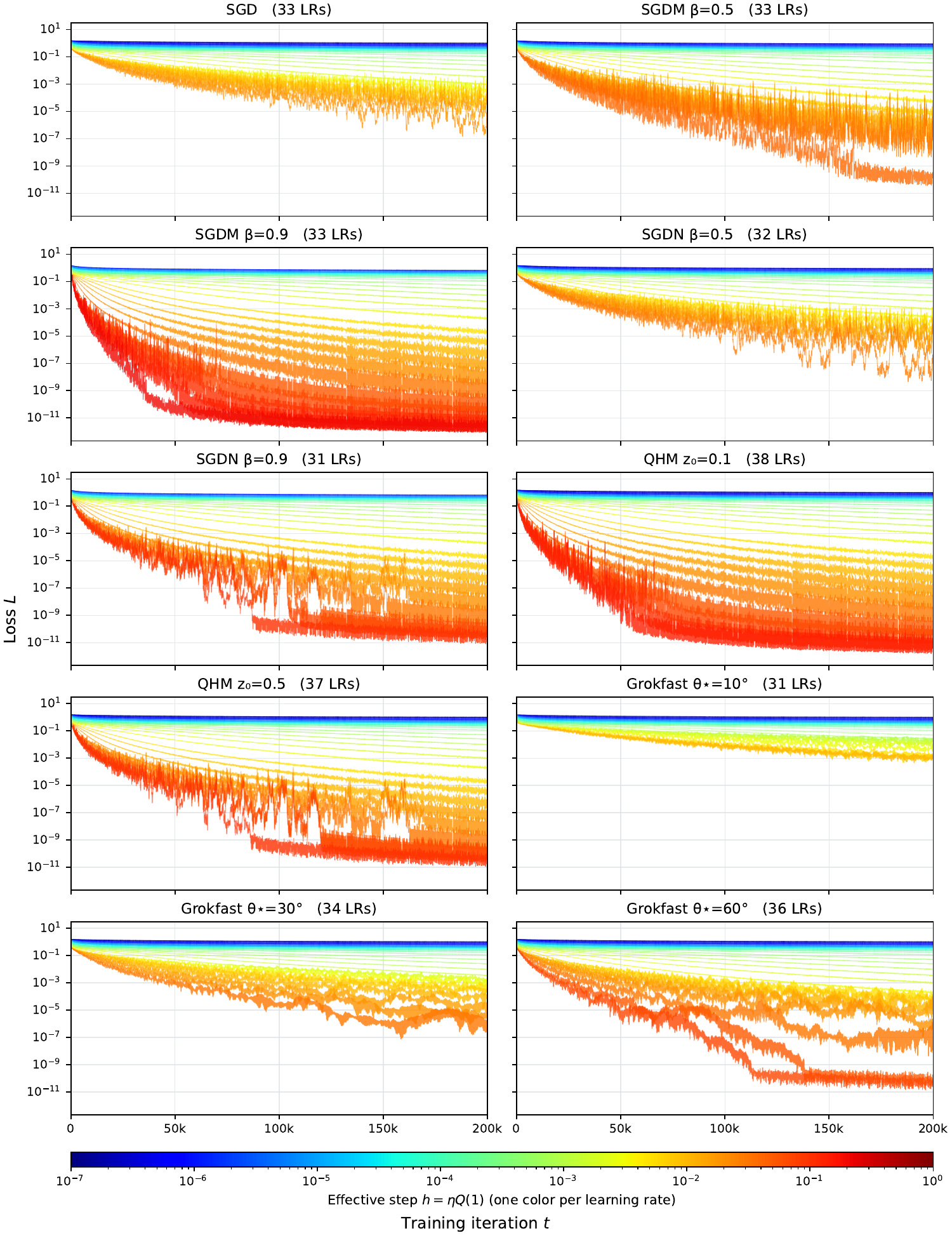}
    \caption{\small \textbf{Individual plots of MLP loss $L$.} Plots are drawn with respect to the training iteration $t$ coordinate. Different learning rates are colored in different colors.}
    \label{fig:mlp-iteration-loss-grid}
\end{figure}

\begin{figure}[p]
    \centering
    \includegraphics[width=\textwidth]{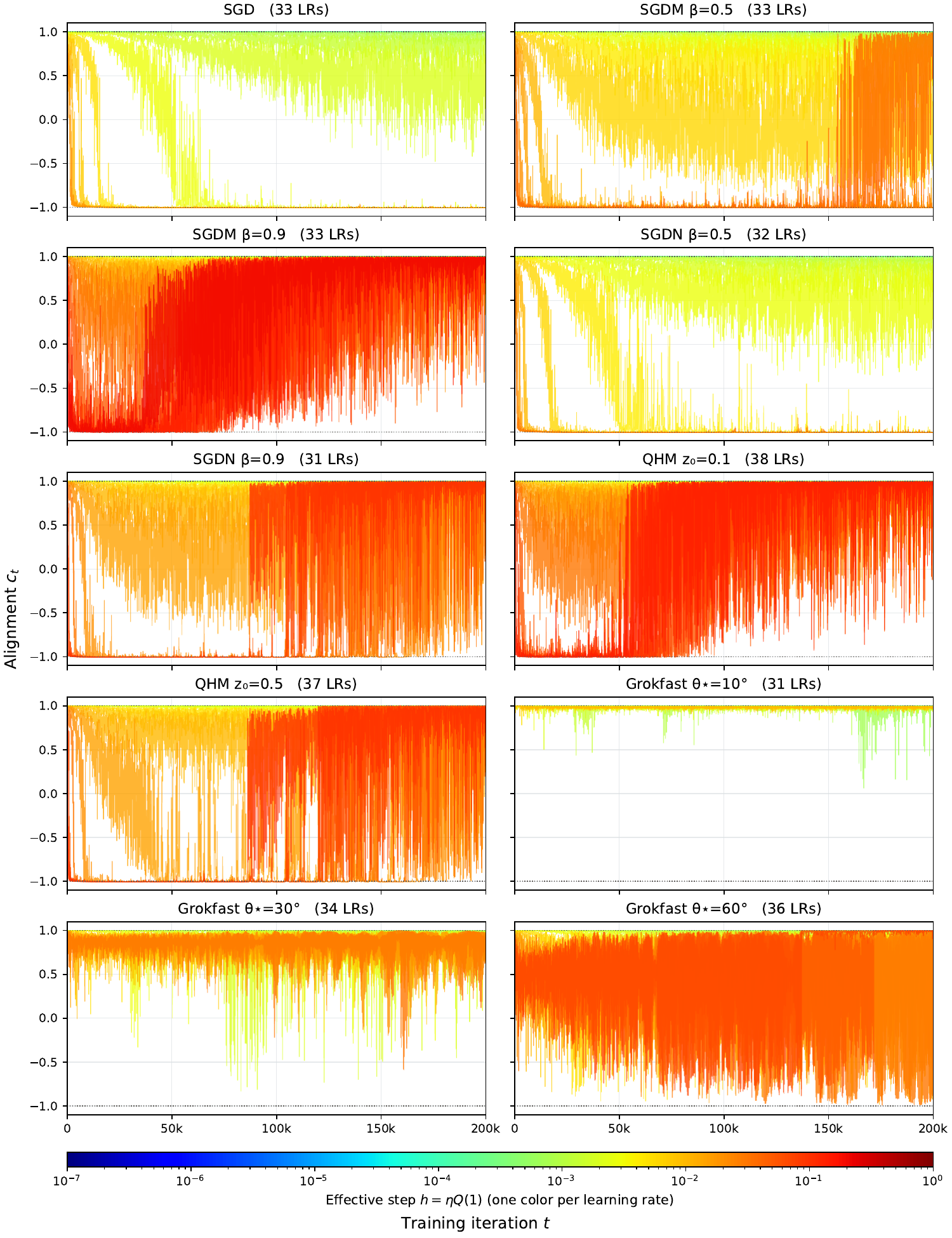}
    \caption{\small \textbf{Individual plots of MLP alignment $c_t$.} Plots are drawn with respect to the training iteration $t$ coordinate. Different learning rates are colored in different colors.}
    \label{fig:mlp-iteration-alignment-grid}
\end{figure}

\begin{figure}[p]
    \centering
    \includegraphics[width=\textwidth]{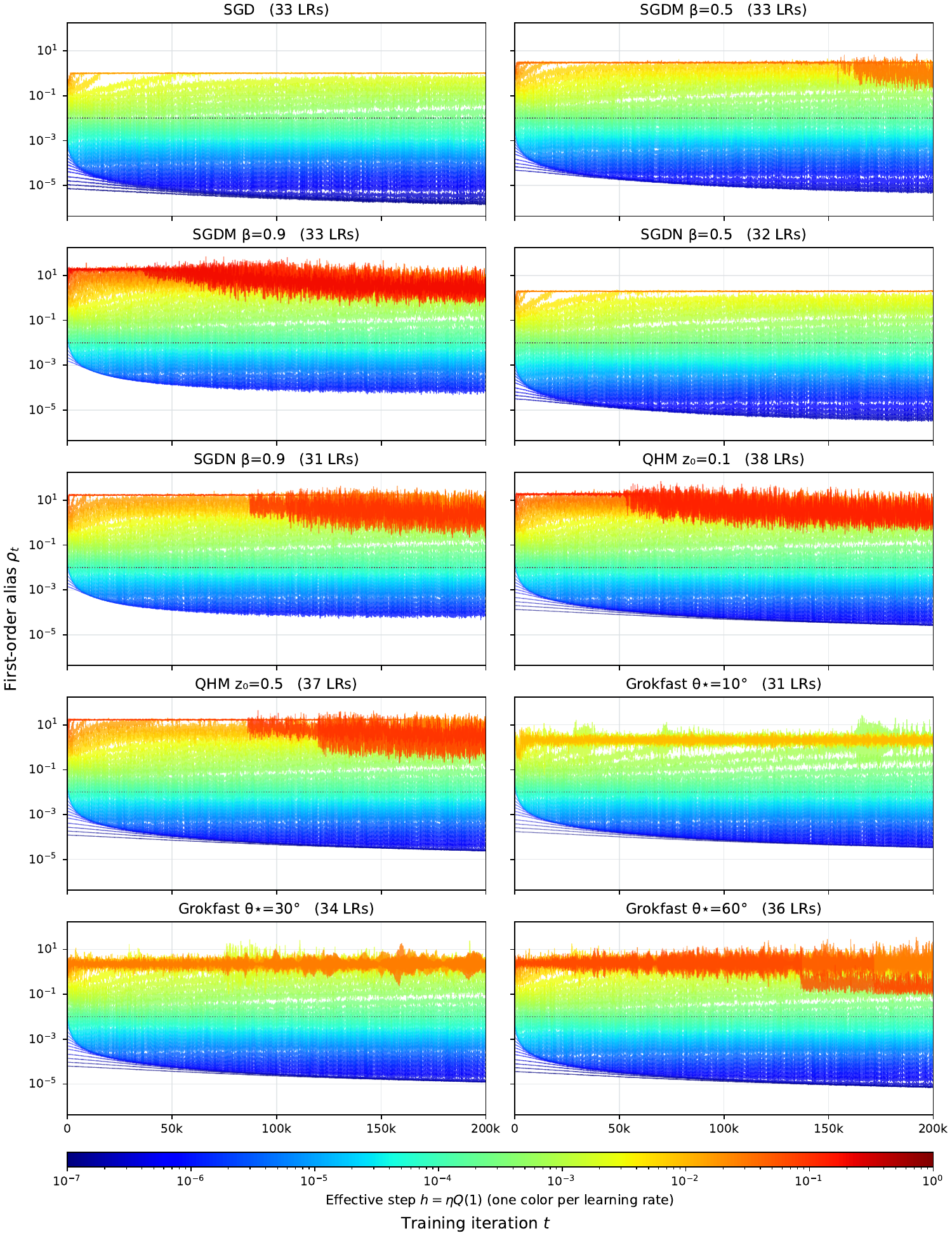}
    \caption{\small \textbf{Individual plots of MLP aliasing $\rho_t$.} Plots are drawn with respect to the training iteration $t$ coordinate. Different learning rates are colored in different colors.}
    \label{fig:mlp-iteration-alias-grid}
\end{figure}

\begin{figure}[p]
    \centering
    \includegraphics[width=\textwidth]{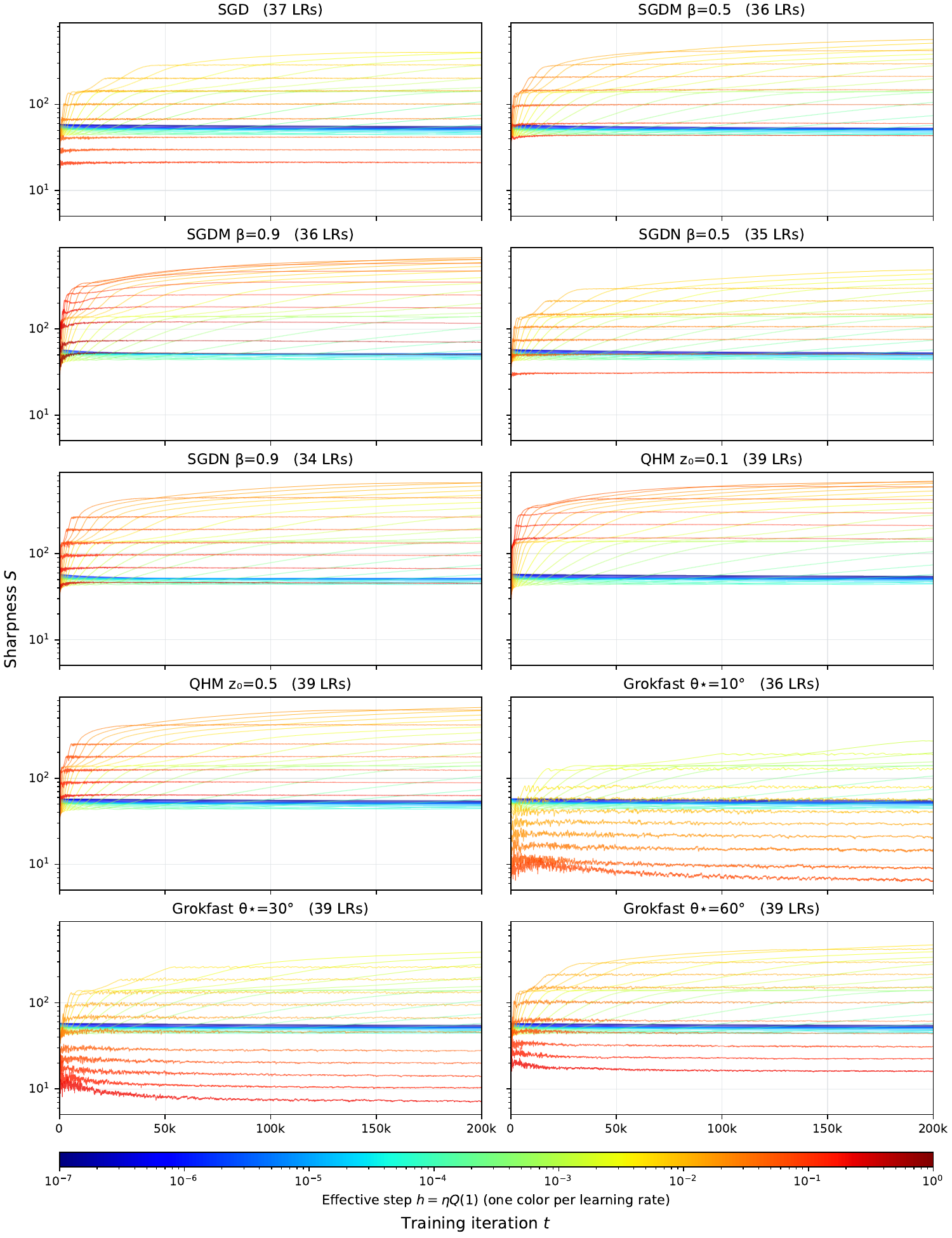}
    \caption{\small \textbf{Individual plots of CNN sharpness $S$.} Plots are drawn with respect to the training iteration $t$ coordinate. Different learning rates are colored in different colors.}
    \label{fig:cnn-iteration-sharpness-grid}
\end{figure}

\begin{figure}[p]
    \centering
    \includegraphics[width=\textwidth]{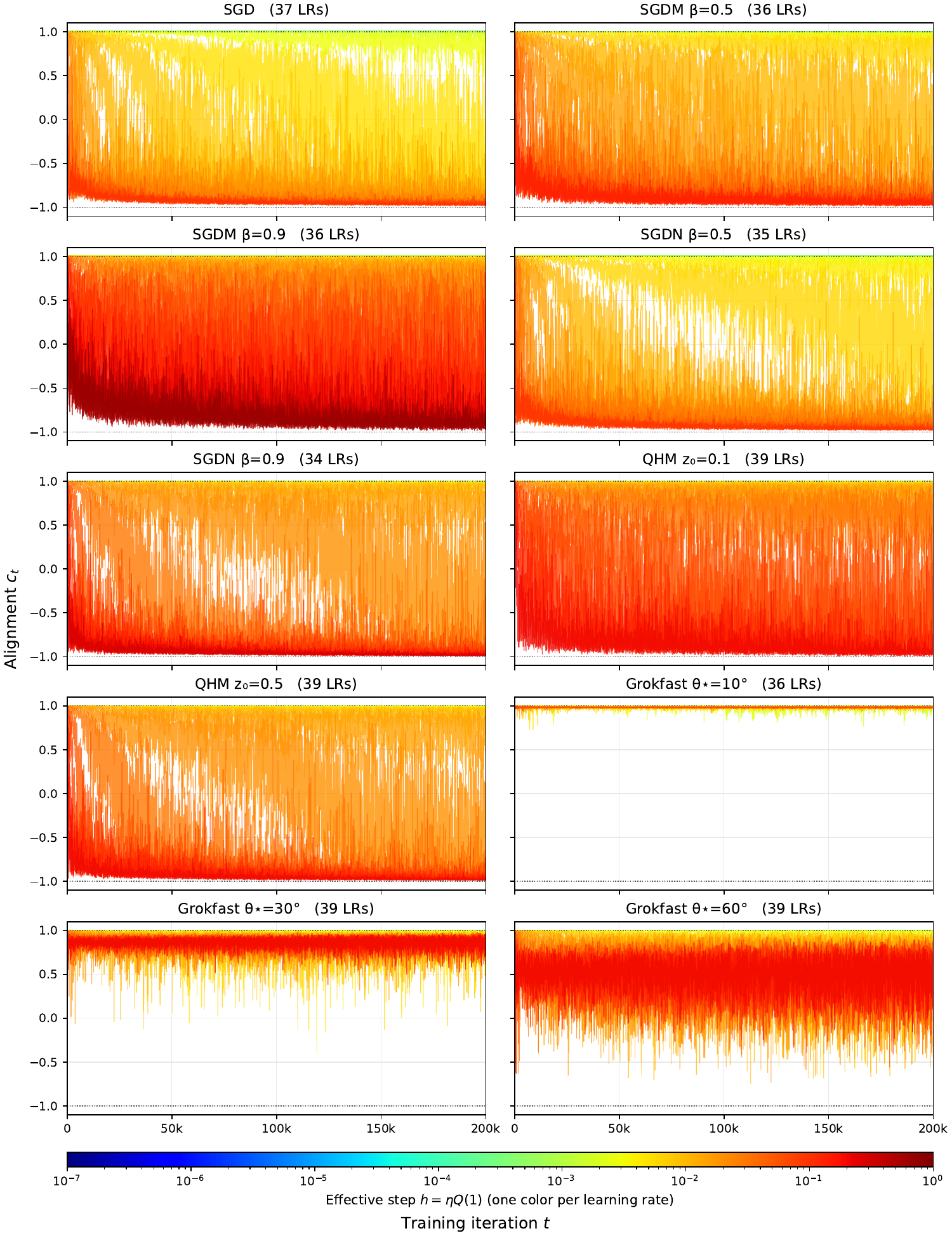}
    \caption{\small \textbf{Individual plots of CNN alignment $c_t$.} Plots are drawn with respect to the training iteration $t$ coordinate. Different learning rates are colored in different colors.}
    \label{fig:cnn-iteration-alignment-grid}
\end{figure}

\begin{figure}[p]
    \centering
    \includegraphics[width=\textwidth]{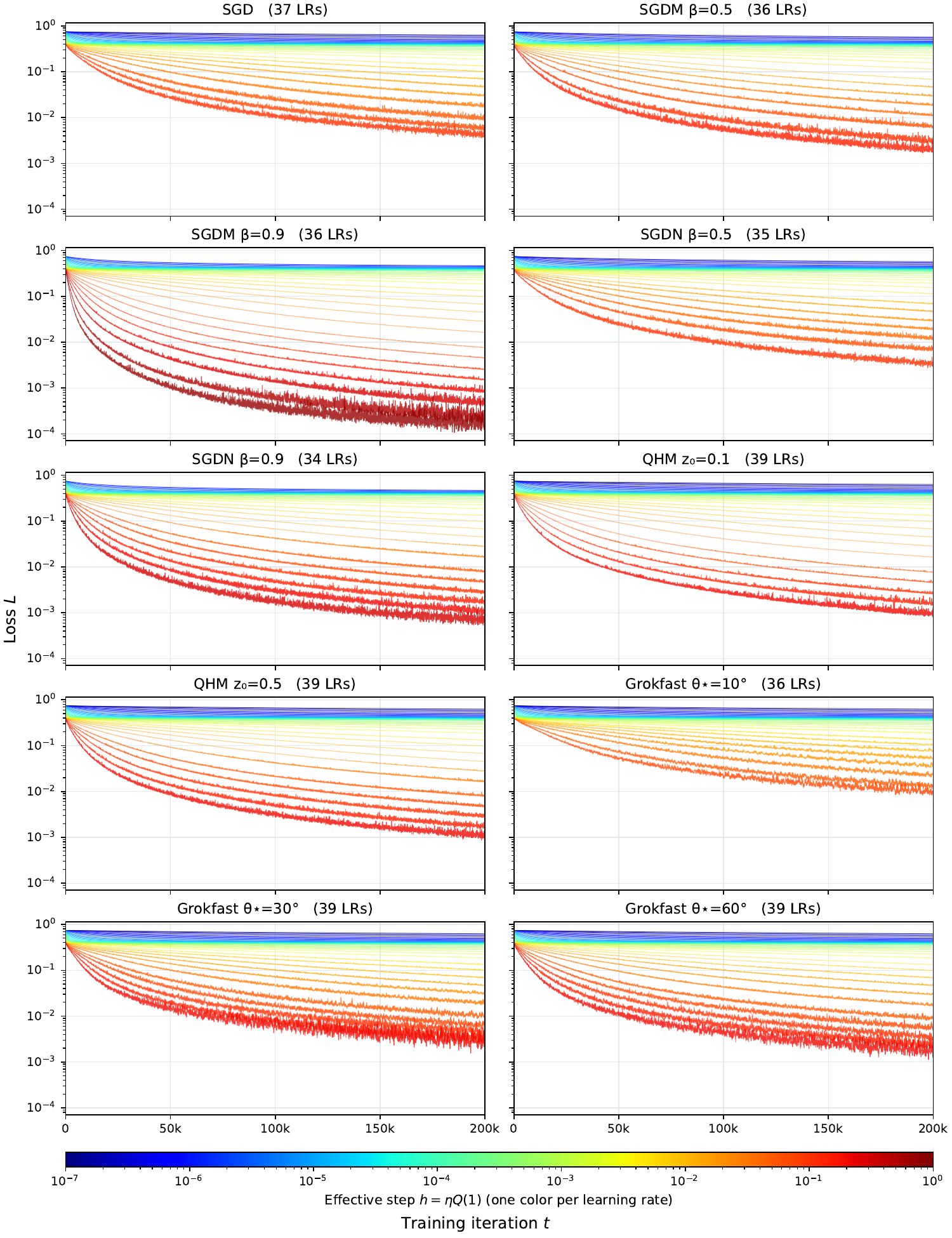}
    \caption{\small \textbf{Individual plots of CNN loss $L$.} Plots are drawn with respect to the training iteration $t$ coordinate. Different learning rates are colored in different colors.}
    \label{fig:cnn-iteration-loss-grid}
\end{figure}

\begin{figure}[p]
    \centering
    \includegraphics[width=\textwidth]{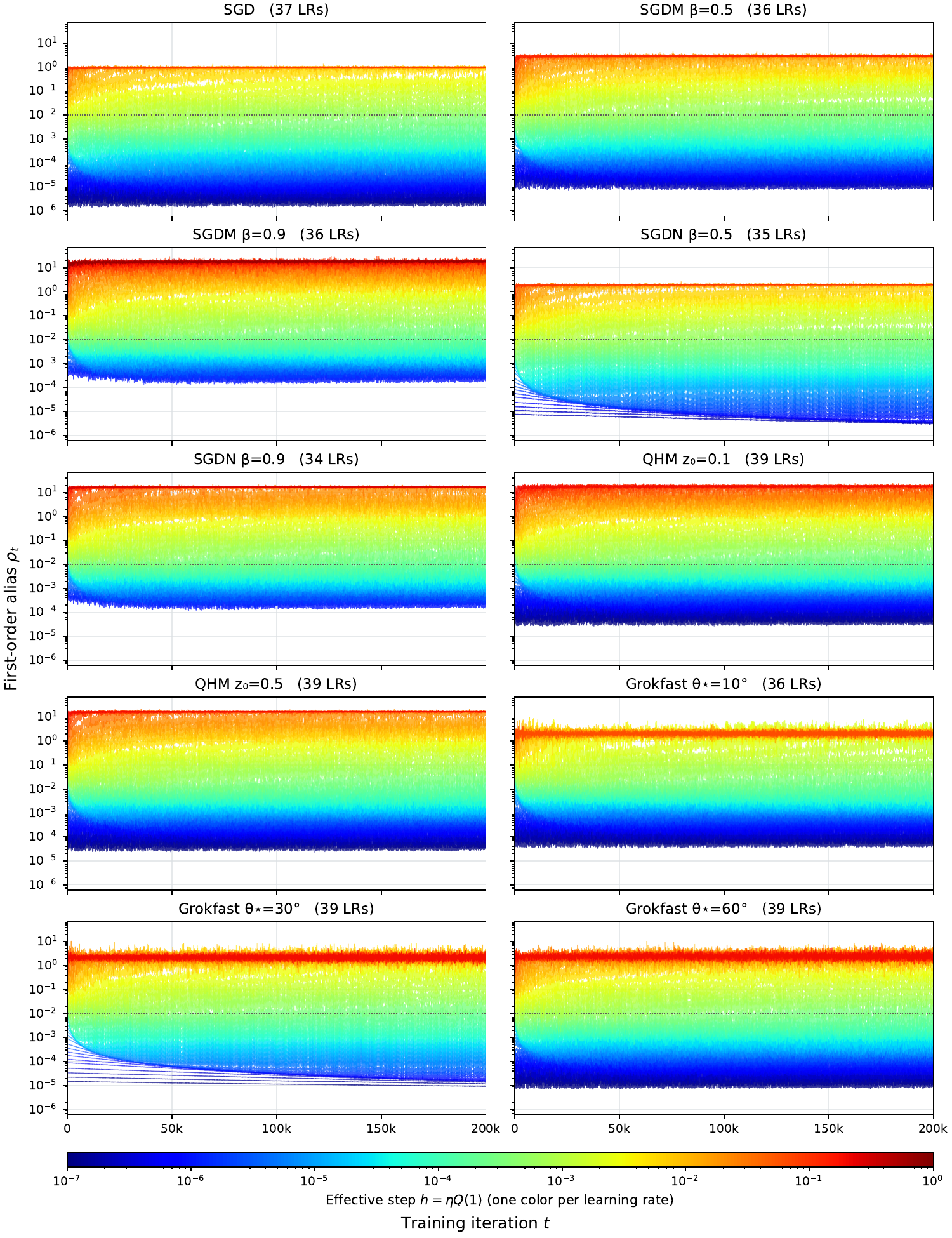}
    \caption{\small \textbf{Individual plots of CNN aliasing $\rho_t$.} Plots are drawn with respect to the training iteration $t$ coordinate. Different learning rates are colored in different colors.}
    \label{fig:cnn-iteration-alias-grid}
\end{figure}

\begin{figure}[p]
    \centering
    \includegraphics[width=\textwidth]{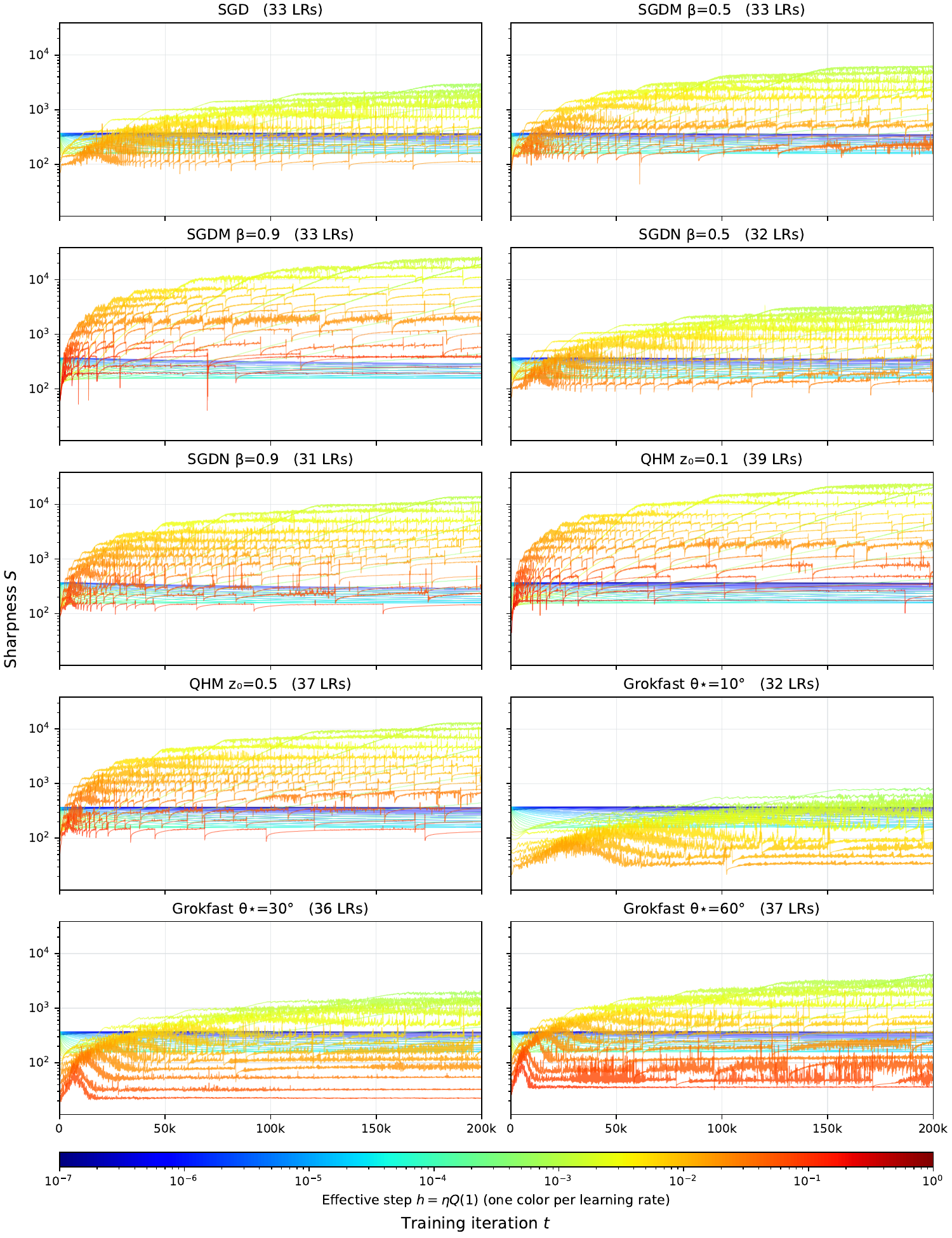}
    \caption{\small \textbf{Individual plots of ViT sharpness $S$.} Plots are drawn with respect to the training iteration $t$ coordinate. Different learning rates are colored in different colors.}
    \label{fig:vit-iteration-sharpness-grid}
\end{figure}

\begin{figure}[p]
    \centering
    \includegraphics[width=\textwidth]{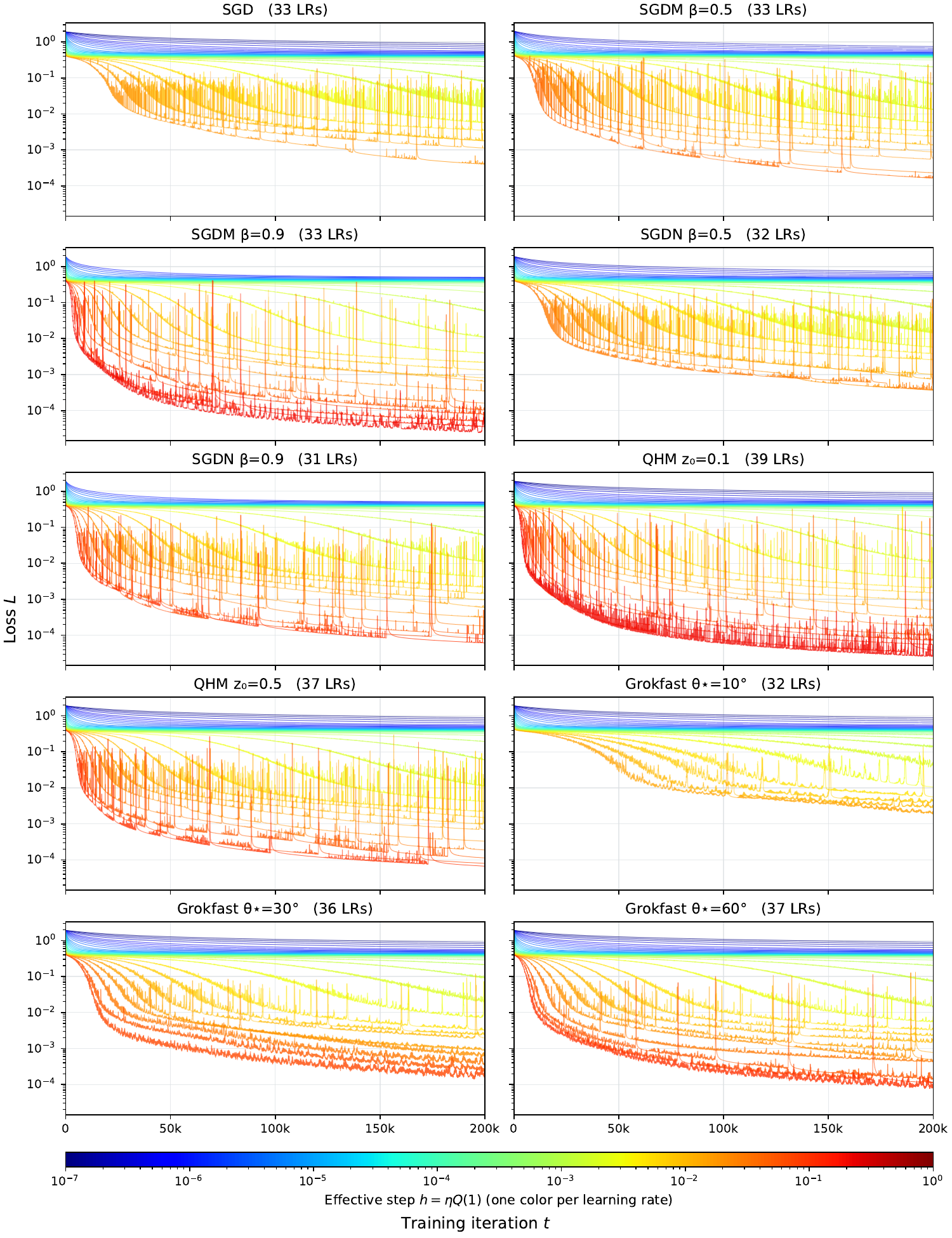}
    \caption{\small \textbf{Individual plots of ViT loss $L$.} Plots are drawn with respect to the training iteration $t$ coordinate. Different learning rates are colored in different colors.}
    \label{fig:vit-iteration-loss-grid}
\end{figure}

\begin{figure}[p]
    \centering
    \includegraphics[width=\textwidth]{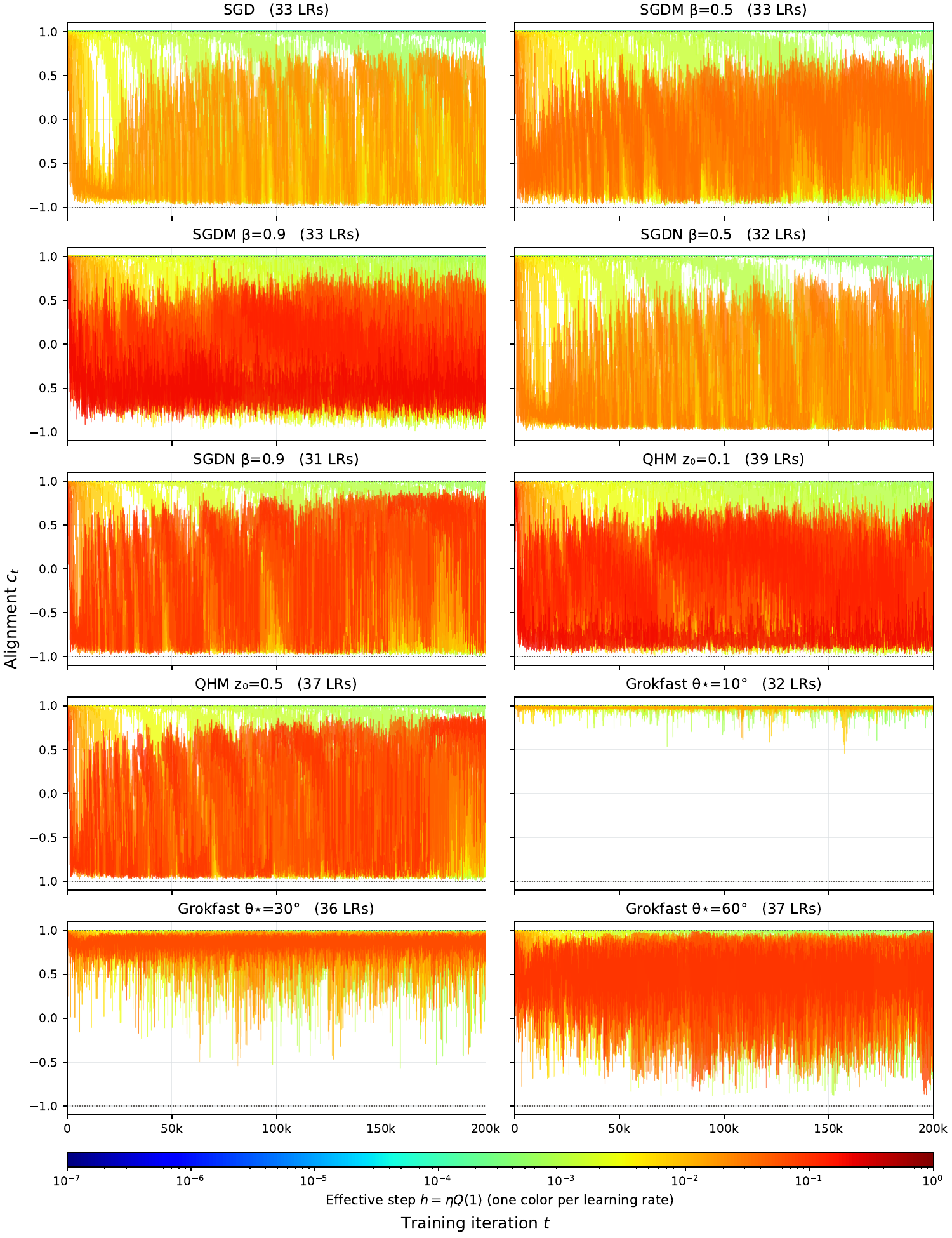}
    \caption{\small \textbf{Individual plots of ViT alignment $c_t$.} Plots are drawn with respect to the training iteration $t$ coordinate. Different learning rates are colored in different colors.}
    \label{fig:vit-iteration-alignment-grid}
\end{figure}

\begin{figure}[p]
    \centering
    \includegraphics[width=\textwidth]{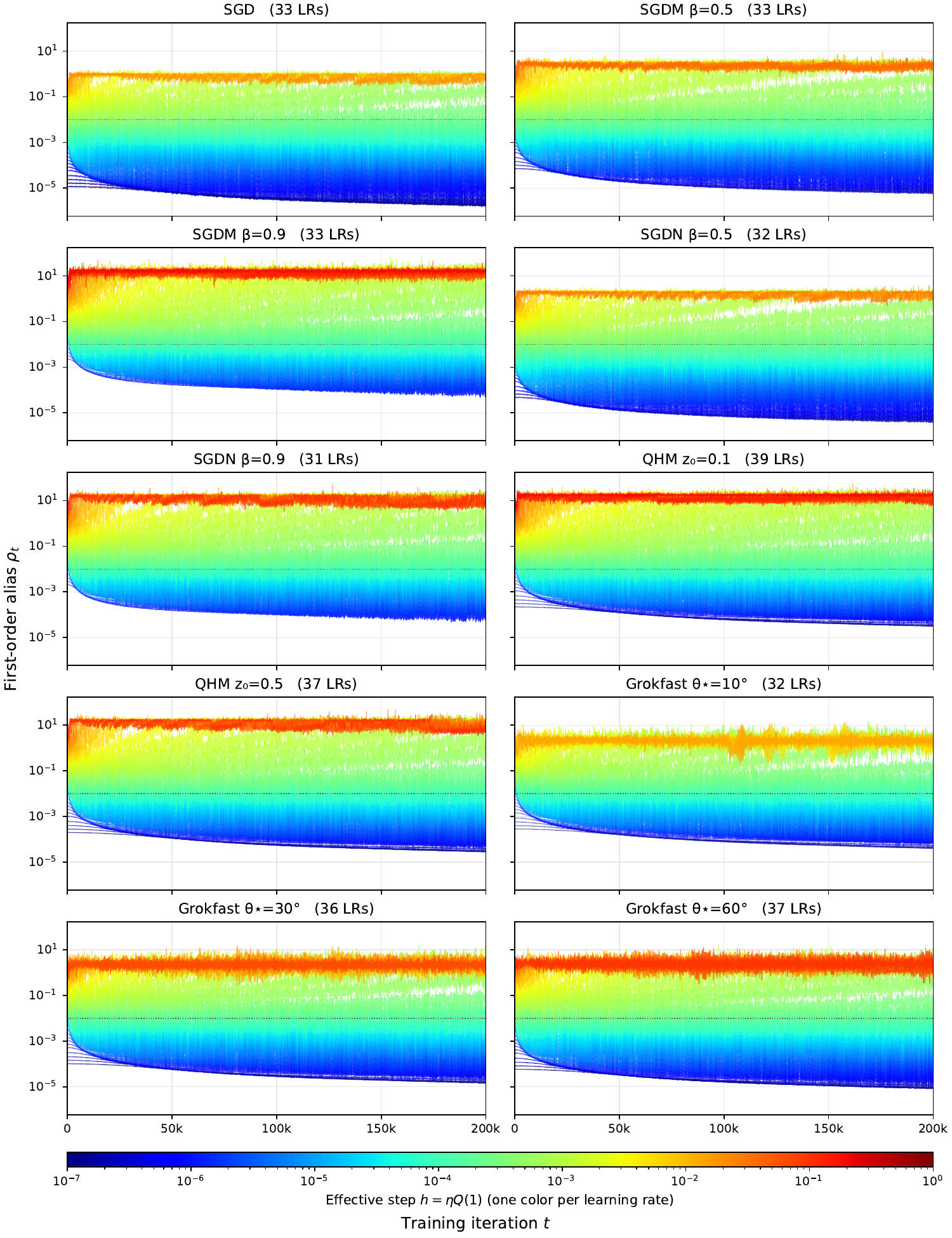}
    \caption{\small \textbf{Individual plots of ViT aliasing $\rho_t$.} Plots are drawn with respect to the training iteration $t$ coordinate. Different learning rates are colored in different colors.}
    \label{fig:vit-iteration-alias-grid}
\end{figure}

\label{sec:exp}

\end{document}